\documentclass[11pt]{article}

\usepackage{rx_macros}

\allowdisplaybreaks

\title{\bf Taming Equilibrium Bias in Risk-Sensitive Multi-Agent Reinforcement Learning}
\author{Yingjie Fei%
\thanks{\texttt{yf275@cornell.edu}}\qquad
Ruitu Xu%
\thanks{Department of Statistics and Data Science, Yale University; \texttt{chris.xu@yale.edu}}}
\date{}

\begin{document}

\maketitle

\begin{abstract}
We study risk-sensitive multi-agent reinforcement learning under general-sum Markov games, where agents optimize the entropic risk measure of rewards with possibly diverse risk preferences. 
We show that using the regret naively adapted from existing literature as a performance metric could induce policies with equilibrium bias that favor the most risk-sensitive agents and overlook the other agents. To address such deficiency of the naive regret, we propose a novel notion of regret, which we call risk-balanced regret, and show through a lower bound that it overcomes the issue of equilibrium bias.
Furthermore, we develop a  self-play algorithm for learning Nash, correlated, and coarse correlated equilibria in risk-sensitive Markov games. We prove that the proposed algorithm attains   near-optimal regret guarantees with respect to the risk-balanced regret.
\end{abstract}

% \begin{ptodo}
%     \begin{enumerate}
%         \item Check proof of upper bound
%         \item Add proof for CE and CCE
%         \item Add proof for lower bound
%         \item Expand related works
%         \item Alternative title: Taming Agent Bias in Heterogeneous Multi-Agent
%         Reinforcement Learning
%     \end{enumerate}
% \end{ptodo}

\section{Introduction}

Recent advancement in reinforcement learning research has witnessed much development on \MARL. However, most of the works focus on risk-neutral agents, which may not be suitable for modeling the real world.
For example, in investment activities, different investors have different risk preferences depending on their roles in the market. 
Some act as  speculators and are  risk-seeking, while others are bound by regulatory constraints and are thus risk-averse. 
Another example is  multi-player online role-playing games, where each of the players can be considered an agent. Whereas some (risk-seeking) players enjoy exploring uncharted regions in the game, others (risk-averse players) prefer to playing in areas that are well explored and come with less uncertainty. 
It is not hard to see that in the above examples, modeling each agent as uniformly risk-neutral is inappropriate. This naturally calls for a more sophisticated modeling framework that takes into account of heterogeneous risk preferences of agents.

In this paper, we study 
% a more realistic multi-agent model, by considering 
the problem of risk-sensitive \MARL under the setting of general-sum \MGs, a more realistic multi-agent model in which the agents may take  different risk preferences.
To that end, we consider $M\ge 1$ agents who maximize the entropic risk measure of rewards, informally defined as
\begin{align*}
    V_m \defeq   \frac{1}{\beta_{m}}\log
    \E
    [e^{\beta_m R_m}],
    \quad \forall m \in [M],
\end{align*}
where $\beta_m$ represents the risk parameter of agent $m$ and $R_m$ is the corresponding reward. For each agent $m$, $\beta_m > 0$ means that the agent prefers more risk (or risk seeking) and  $\beta_m < 0$ indicates that the agent favors less risk (or risk averse); when $\beta_m \to 0$, the agent tends to be  risk-neutral and maximizes the expected reward $\E[R_m]$.
We aim to learn equilibria of the underlying \MG in an online fashion, without access to  knowledge or simulator of the unknown transition kernels. 
% We provide novel algorithms achieving equilibrium solutions in the games and their non-asymptotic guarantees. 
% Our algorithms, paired with the fresh metric of regret, aim to allocate resource commensurate to the needs of each agent, instead of catering to a particular agent.

In the online setting, a standard metric for quantifying the performance of an algorithm is \textit{regret}, which computes the cumulative difference between the best possible utility and the utility attained by the algorithm.
Unfortunately, the formulation of regret naively adapted from the risk-neutral setting fails to represent a suitable performance metric under the risk-sensitive setting: we show that the naive regret is dominated by the sub-optimality of the most risk-sensitive agents, demonstrated through a lower bound.
% that has to be incurred by any algorithms. 
Such deficiency may lead to algorithms that appear nearly optimal but generate policies that suffer \textit{equilibrium bias}, by which the most risk-sensitive agents are falsely favored against the remaining agents, and exponential sub-optimality 
% exponentially than that of nearly optimal algorithms 
is incurred for all agents except for the most risk-sensitive ones. 
% Such algorithms necessarily incurs equilibrium bias, which 

% This result implies several deficiencies of this notion of regret. First, for certain algorithms that only learn for agents of highest risk sensitivity and do not learn at all for other agents, it takes exponential time to detect their sub-optimality. This occurs even in the ideal scenario where we have access to that optimal value of each agent and are able to compute regret exactly. 
% In other cases, it is even impossible to detect such sub-optimality. 
% These deficiencies may impede the design and test of efficient algorithms. 

To address the issue of equilibrium bias, we propose a novel definition of regret tailored to the risk-sensitive setting, which we name as  risk-balanced regret. The risk-balanced regret takes into account of possibly diverse risk sensitivity of all agents and treats all agent in a symmetric way despite their different risk preferences. We prove a lower bound based on this new notion of regret, which suggests that it addresses the problem of equilibrium bias suffered by the naive regret.
In addition, we propose a novel self-play algorithm that learns Nash, correlated, and coarse correlated equilibria of the general-sum \MG involving multiple risk-sensitive agents, based on value iteration and optimistic exploration. We prove that the proposed algorithm achieves a nearly optimal upper bound for the risk-balanced regret. 
To the best of our knowledge, this work provides the first finite-sample guarantees in risk-sensitive \MARL based on the entropic risk measure.
% Our proof involves carefully constructing upper bounds of errors that adapt to different risk preferences of all agents.  

In summary, our work presents the following theoretical contributions.
\begin{itemize}
    \item 
    We consider risk-sensitive \MARL under general-sum \MGs where each agent may have heterogeneous risk preferences. We identify the pitfall of a regret metric naively adapted from the risk-neutral setting: it induces policies that incur equilibrium bias by improperly biasing towards agents with the largest risk sensitivity among the agent population.
    % , and it impedes detection of sub-optimal algorithms even if we had access to optimal values of all agents.
    \item 
    We propose a novel notion of regret (risk-balanced regret) for quantifying the performance of online learning algorithms for risk-sensitive agents, which accounts for the risk sensitivity of each agent and overcomes the issue of equilibrium bias.
    \item 
    We develop a novel  self-play algorithm for learning  Nash, correlated, and coarse correlated equilibria, and we prove that the proposed algorithm achieves a nearly optimal rate w.r.t.\ the risk-balanced regret compared to the lower bound.
    % . We show that, under the new regret metric, our propose algorithm achieves a $\widetilde{O}(K^{1/2})$ regret, for which we provide a nearly matching lower bound. 
    % To the best of our knowledge, this work provides the first finite-sample guarantees in risk-sensitive \MARL based on the entropic risk measure.
\end{itemize}

\section{Related Work}
\MARL \citep{littman1994markov} has been a central topic in the literature of reinforcement learning, with an extensive line of work  devoted to studying risk-neutral settings and finite-sample guarantees. 
% \MG is a natural generalization of \MDP from single-agent \RL into \MARL framework, where each agent takes actions to maximize its total utility \citep{littman1994markov}. 
% Self-play algorithms for \MGs with finite-sample results have been studied extensively in the past few years. 
% In particular,
For example, 
theoretical studies have investigated two-player zero-sum \MGs and a number of near-optimal self-play algorithms \citep{xie2020learning,bai2020near,bai2020provable}; a nearly minimax optimal regret has also become available for two-player zero-sum \MGs under a linear kernel setup \citep{chen2021almost}. Further,  zero- and general-sum \MGs are studied by \citet{liu2020sharp} and \citet{zhang2020model} with nearly optimal sample complexity bounds, while \citet{dubey2021provably} considers multi-agent cooperative games with communication budget. The work of \citet{tian2020provably} develops an algorithm for multi-player general-sum \MGs that is agnostic of other agents' actions.

% Especially, 

% While the above works focus on risk-neutral settings,
Risk-sensitive \MARL  (especially based on the entropic risk measure) has also been investigated in a thread of literature. The paper 
\citet{klompstra2000nash} characterizes Nash equilibria of risk-sensitive two-player general-sum control games, 
% which can be regarded as 
a special case of general-sum \MGs.
The works of \citet{basu2012zero,basu2014zero,bauerle2017zero} prove the existence of equilibria for discounted risk-sensitive two-player zero-sum games, and \citet{cavazos2019vanishing} establishes the convergence of risk-sensitive value functions for zero-sum finite games. 
Moreover, the work  \citet{huang2019model} considers non-cooperative muti-agent games with risk-averse agents and proposes an algorithm with  almost-sure convergence results. 
% assuming access to a simulator. 
The author of \citet{wei2019nonzero}  proves the existence of Nash equilibria for risk-sensitive multi-agent general-sum games. Our work differs from these works, which mostly focus on asymptotic results, as we provide finite-sample guarantees for risk-sensitive \MARL.
% without assuming knowledge or simulator to game transition.

\section{Preliminaries}

In this section, we introduce the  setup of general-sum \MG with multiple agents under entropic risk measure, starting with notations.

\subsection{Notation}

We denote $[n] \defeq 1,2,\ldots,n$ for any positive integer $n$. For any functions $f$ and $g$ defined on $\XX$, we let $f\leq g$ denote $f(x)\leq g(x)$ for all $x\in\XX$ and  $f\geq g$ to be defined similarly. For any  functions $f$ and $g$ defined on $\XX\subseteq\ZZ_+$, $f(n)\lesssim g(n)$ denotes $f(n)\leq C\cdot g(n)$ for every $n\in\XX$ with some universal constant $C>0$ and $f(n)\gtrsim g(n)$ is defined similarly. The notation $f(n)\asymp g(n)$ means that we have $f(n)\lesssim g(n)$ as well as $f(n)\gtrsim g(n)$.
% We further denote $1,\ldots,n$ as $[n]$ for any $n\in\ZZ_+$.
We let $\log(\cdot)$ denote natural logarithm with base $e$ unless specified otherwise. 
The expression $\mathrm{polylog}(v_1,\ldots,v_n)$ is defined as $c_0\prod_{i\in[n]}\log(v_i)^{c_i}$ for some  universal constants $\{c_i\}_{i \ge 0}$ and variables $\{v_i\}_{i \in [n]}$. 
We also define $\mathrm{poly}(v_1,\ldots,v_n) \defeq c_0 \prod_{i\in[n]}v_i^{c_i}$.
Furthermore, we use $\widetilde O(\cdot)$ to denote $O(\cdot\ \mathrm{polylog}(\cdot))$; we define $\widetilde\Omega(\cdot)$ and $\widetilde\omega(\cdot)$ similarly. 

\subsection{Problem Setup}
% \paragraph{Problem Setup.}
We consider a risk-sensitive general-sum \MG with $M$ agents in the tabular setting, which can be represented as  $\mathrm{MG}(H,K,\cS,\{\cA_{m}\}_{m=1}^{M},\calP,\{r_{m}\}_{m=1}^{M},\{\beta_{m}\}_{m=1}^{M})$. Here, $H$ denotes the horizon (or the  number of steps in each episode), $\calS$ denotes the state space of size $S \defeq |\calS|$. For each agent $m\in[M]$, $\cA_{m}$ denotes the action space available to the agent with size $A_m \defeq |\calA_m|$, and for the convenience of notation, we further define $\cA\defeq\prod_{m\in[M]}\cA_{m}$ to be the action space of all $M$ agents, with cardinality $A \defeq \prod_{m\in[M]} A_m$. In particluar, we use $a_{h,m}$ to denote the action that agent $m$ takes at step $h$ and $a_h\defeq (a_{h,1},\ldots,a_{h,M})$ to denote the actions of all $M$ agents at step $h$.
Each agent has reward functions $r_m \defeq \{r_{h,m}\}_{h=1}^H$, where $r_{h,m}:\calS\times\calA_m\to[0,1]$ is the reward function at each step $h$. 
In the risk-sensitive setting, each agent $m$ takes a risk parameter $\beta_{m}\ne0$ representing its risk preference; the agent is risk-seeking if $\beta_m>0$ and risk-averse if $\beta_m<0$. We specify the way $\{\beta_m\}$ are involved in the agents' objectives in the next section.
The transition kernels $\calP = \{\calP_h\}_{h=1}^H$ provide the transition probability of the underlying \MG, \ie, the game transitions to $s'\in\calS$ with probability $\calP_h(s'\given s,a)$ at step $h$ given the current state-action pair $(s,a)\in\calS\times\calA$. The total number of episodes is represented by $K$.
In each episode $k$, we assume that the game begins at a fixed initial state $s^k_1 = s_1$.
% without loss of generality.

% \paragraph{Interaction protocol.}

In this paper, 
we study the self-play setting, and we outline the interaction protocol between the learning algorithm and agents in \cref{alg:proto}.
% Given any \MG, \cref{alg:proto} outlines the interaction protocol of a learning algorithm on a set of agents. 
On the agent side, at each step $h\in[H]$ of an episode $k \in [K]$, each agent $m$ observes the current state $s_h$ and  chooses an action $a_{h,m}$ simultaneously with other agents. Then 
it 
% observes actions $a_h$ of all agents and 
receives its reward $r_{h,m}(s_h,a_h)$ and the game state transitions to the next state $s_{h+1}\sim\calP_h(\cdot\given s_h,a_h)$. An episode ends after the game reaches step $H+1$. 
We remark that the learning algorithm does not have access to  either the transition kernel or reward functions. 
% The detailed description of our algorithm is laid out in \cref{sec:algo}.

\begin{algorithm}
\begin{algorithmic}[1]
    % \State The learning algorithm is aware of the action space and risk parameter of individual agents, but not the transition kernel of the \MDP and the reward function of each agent
    \For{each episode $k$}
        \For{each step $h$}
            \State The learning algorithm computes policies for each agent
            \State Each agent plays action from the policies prescribed by the learning algorithm
            \State Each agent receives the next state and its  reward (but not reward function)
            \State Information obtained by each agent is sent to the learning algorithm
        \EndFor
    \EndFor
\end{algorithmic}
\caption{Interaction protocol of self-play  \label{alg:proto}}
\end{algorithm}

\subsection{Policy and Value Functions}
% \paragraph{Policy and value functions.}

For each agent $m\in[M]$, we define 
% $\pi_{m} \defeq \{\pi_{h,m}\}_{h\in[H]}$
$\pi_{h,m}:\calS\to\Delta_{\calA_m}$
as its policy at step $h$ that 
% , where $\pi_{h,m}:\calS\to\Delta_{\calA_m}$ 
maps from each state to a distribution on the action space $\calA_m$. 
% Given individual policies $\{\pi_m\}_{m\in[M]}$ for all agents, we further denote the product policy by $\pi\defeq \prod_{m=1}^M \pi_m$. With a slight abuse of notation, we also 
We denote 
% $\pi_h\defeq \{\pi_{h,m}\}_{m=1}^M$ 
$\pi_h$ 
as the joint policy for all $M$ agents, where $\pi_h:\calS\to\Delta_\calA$ maps from each state to a distribution on the joint action space $\calA$ at each step $h$. When the decomposition $\pi_h\defeq \prod_{m=1}^M \pi_{h,m}$  exists, we say that $\pi_h$ is a product policy.

% Note that here $\pi$ represents both joint policy and product policy, with its exact notion possibly changing from line to line.
% ; this is for the convenience of notation, and the type of policy it stands for is going to be clear from the context.
Given any policy $\pi=\{\pi_h\}$ and risk parameters $\{\beta_m\}$ of the $M$ agents, we define the value function $V_{h,m}^{\pi}$ for agent $m$ at step $h$ as the expected cumulative reward under the entropic risk measure for each state $s\in\calS$, \ie,
\begin{align*}
    V_{h,m}^{\pi}(s) \defeq \frac{1}{\beta_{m}}\log\left\{ \E_{\pi}\left[e^{\beta_m  \sum_{i=h}^H r_{i,m}(s_{i},a_{i})}\ \bigg|\ s_{h}=s\right]\right\},
\end{align*}
and the corresponding action-value function can be similarly defined for all state-action pairs $(s,a)\in\calS\times\calA$, \ie,
\begin{align*}
    Q_{h,m}^\pi(s,a) & \defeq \frac{1}{\beta_m} \log\Big\{  \expect_\pi\Big[e^{\beta_m  \sum_{i=h}^H r_{i,m}(s_{i},\pi_{i}(s_{i}))}\ \Big|\ s_h=s, a_h=a \Big]  \Big\}.
\end{align*}
% This definition of value function has been considered for risk-sensitive \gls*{RL} in \citet{fei2021exponential,fei2022cascaded}. 
We remark that $\beta_m>0$ represents a risk-seeking agent and $\beta_m < 0$ represents a risk-averse agent; the agent tends to be   risk-neutral as $\beta_m\to 0$.

\subsection{Equilibrium}
% \paragraph{Equilibrium.}

We aim to learn a policy for all agents through $K$ episodes of interactions with the environment, so as to reach certain equilibria for the given general-sum \glspl*{MG}. 
In this paper, we consider three types of equilibria: \NE, \gls*{CE}, and \gls*{CCE}. 
Before we provide their definitions, we first set some additional notations.
For any policy $\pi$, we denote $\pi_{-m}$ to be the joint policy of all agents except agent $m$. We say that $\pi_m^{*}(\pi_{-m})$ is a best response policy  for agent $m$ given $\pi_{-m}$, if it holds that $V_{h,m}^{\pi_m^{*}(\pi_{-m}),\pi_{-m}}(s) = \sup_{\nu} V_{h,m}^{\nu,\pi_{-m}}(s)$ for all $(h,s)\in [H]\times\cS$. 
We denote the value function under the best response of agent $m$ as $V_{h,m}^{*,\pi_{-m}}\defeq V_{h,m}^{\pi_m^{*}(\pi_{-m}),\pi_{-m}}$ for short. The action-value function for the best response $Q_{h,m}^{\pi_m^{*}(\pi_{-m}),\pi_{-m}}$ as well as its shorthand  $Q_{h,m}^{*,\pi_{-m}}$ is similarly defined for all agents $m\in[M]$.

An equilibrium implies that altering the policy of any single agent alone cannot improve its utility.
We say that a product policy $\pi$ is a \NE if the maximum payoff difference over all agents is zero, \ie, 
% \begin{align*}
$    \max_{m\in[M]} (V_{1,m}^{*,\pi_{-m}} - V_{1,m}^{\pi})(s_1) = 0.$
% \end{align*}
Moreover, we say that a joint policy $\pi$ (not necessarily a product policy) is a \gls*{CCE} if $\max_{m\in[M]} (V_{1,m}^{*,\pi_{-m}} - V_{1,m}^{\pi})(s_1) = 0$. 
With regard to \gls*{CE}, we define strategy modification $\psi_m\defeq \{\psi_{h,s,m}\}_{(h,s)\in[H]\times\calS}$, where $\psi_{h,s,m}$ maps from $\calA_m$ to itself for each agent $m\in[M]$, state $s\in\calS$, and step $h\in[H]$. Let $\Psi_m$ denote the set of all possible strategy modifications $\psi_m$ available to agent $m$. When we apply a strategy modification $\psi_{m}\in\Psi_m$ to a policy $\pi$,  if $(a_{h,1},\ldots,a_{h,M})$ is the  $\pi$-induced actions for all agents $m\in[M]$ given state $s$ and step $h$, then the modified policy $\psi_m \diamond \pi$ plays the modified joint action $(a_{h,1},\ldots,a_{h,m-1},\psi_{h,s,m}(a_{h,m}),a_{h,m+1},\ldots,a_{h,M})$.
Given the definition of strategy modification, we say that a joint policy $\pi$ is a \gls*{CE} if it satisfies that $\max_{m\in[M]}\max_{\psi\in\Psi_m} (V_{1,m}^{\psi\diamond\pi} - V_{1,m}^{\pi})(s_1) = 0$,  which means that the best strategy modification for any single agent cannot improve its utility. 
We remark that the concepts of NE, CE and CCE are closely related. Indeed, it can be shown that CCE is a sub-class of CE, which is in turn a sub-class of NE \citep{nisan2007algorithmic}.
Since NE always exists, CE and CCE  always exist as well.

% Therefore, the standard regret for risk-neutral \MARL is defined as the sum of the maximum of the expected payoff differences in policy $\pi^k$ over all episodes $k\in[K]$, \eg, the standard regret for \NE is  
% \begin{align*}
%     \overline{\reg}_{\mathsf{NE}}(K) = \sum_{k\in[K]}\max_{m\in[M]}(V_{1,m}^{*,\pi_{-m}^{k}}-V_{1,m}^{\pi^{k}})(s_{1}).
% \end{align*}
% The corresponding definitions of $\overline{\reg}_{\mathsf{CCE}}(K)$ and $\overline{\reg}_{\mathsf{CE}}(K)$ follow the definition of \gls*{CCE} and \gls*{CE} accordingly. In the rest of the paper, we are only going to present the results for \NE to avoid redundancy, and similar results in Appendix also hold for the other two equilibria.

\section{Regret and Equilibrium Bias}\label{sec:regret}

Before diving into our algorithm and regret analysis, we first discuss some of the pitfalls of standard regrets as a natural generalization of their risk-neutral counterparts, \ie, $\overline{\mathrm{Regret}}_{\mathsf{NE}}(K)$, $\overline{\mathrm{Regret}}_{\mathsf{CCE}}(K)$, and $\overline{\mathrm{Regret}}_{\mathsf{CE}}(K)$. 

\subsection{A Naive Definition of Regret and Its Pitfalls}\label{sec:pitfall}

% For simplicity, let us focus on \NE in this section and we note that the same reasoning also applies to \CE and \CCE. 
In existing literature, most research has focused on   risk-neutral \MARL, where algorithm performance is measured by regret: 
\begin{align*}
    \text{RN-Regret}_{\mathsf{NE}}(K) \defeq \sum_{k\in[K]}\max_{m\in[M]}(\widetilde{V}_{1,m}^{*,\pi_{-m}^{k}}-\widetilde{V}_{1,m}^{\pi^{k}})(s_{1}),
\end{align*}
In the above,  $\widetilde{V}^{\pi}_{h,m} (s) \defeq \expect_\pi [\sum_{i=h}^H r_{i,m}(s_{i},\pi_{i}(s_{i})) \given s_h=s ]$ is the risk-neutral value function.
A natural extension of this regret to the risk-sensitive setting would be to replace the risk-neutral value functions by their risk-sensitive counterparts, as in
\begin{align}
    \overline{\reg}_{\mathsf{NE}}(K) \defeq \sum_{k\in[K]}\max_{m\in[M]}(V_{1,m}^{*,\pi_{-m}^{k}}-V_{1,m}^{\pi^{k}})(s_{1}). \label{eq:naive_regret}
\end{align}

Similar definitions can be made for \CE and \CCE. 
It can be seen that the definition in \eqref{eq:naive_regret} generalizes the regret for risk-sensitive RL studied in the single-agent setting \citep{fei2020risk,fei2021exponential,fei2022cascaded} to the multi-agent setting. 

Unfortunately, the naive regret definition in \eqref{eq:naive_regret} has its shortcomings. Before we discuss them in details, we introduce some notations.
Let $m_* \defeq \argmax_{m\in[M]}|\beta_m|$ denote the index of the agent that is the most risk-sensitive among all $[M]$ agents, and if necessary, we may break the tie in an arbitrary way. We also set $\beta_* \defeq \beta_{m_*}$ to be the risk parameter of agent $m_*$. 
% We can pick any of them if there are more than one of such agents. 
% Without loss of generality, we consider a standard regret defined in risk-neutral \MARL, \eg,
% \begin{align*}
%     \overline{\reg}_{\mathsf{NE}}(K) = \sum_{k\in[K]}\max_{m\in[M]}(V_{1,m}^{*,\pi_{-m}^{k}}-V_{1,m}^{\pi^{k}})(s_{1}),
% \end{align*}
% and we are able to obtain a theoretical lower bound in \cref{thm:reg_lower_naive} for any algorithm under the regret. 
For notational convenience, we define a risk-dependent factor $\Phi_u(\beta)$ for any $\beta \ne 0$ 
and $u > 0$ as
\begin{align*}
    \Phi_{u}(\beta)\coloneqq\frac{1}{\left|\beta\right|u}(e^{\left|\beta\right|u}-1).
\end{align*}
% which is going to appear frequently in the rest of this paper. 
Note that $\Phi_u(\beta)$ is an even function in $\beta$ and increases exponentially in $|\beta|$ and $u$.
To showcase the deficiency of the naive definition of regret in \eqref{eq:naive_regret}, we start by presenting a lower bound for the regret in the following theorem.

\begin{theorem}\label{thm:reg_lower_naive}
    For $H \geq 8$, $K \geq \max\{16e^{|\beta_*|(H-1)}, 16H\}$, and $\log\log K \gtrsim |\beta_*|(H-1)$, there exists an \MG such that any algorithm obeys 
    \begin{align*}
        \E\big[\overline{\reg}_{\mathsf{NE}}(K)\big] 
        % = \max_{m\in[M]}\Phi_{H}(\beta_{m})\cdot \widetilde{\Omega}(\sqrt{K H^2}) 
        =  \Phi_{H}(\beta_{*})\cdot \widetilde{\Omega}(\sqrt{K H^2}).
    \end{align*}
    The same bound holds for $\E\big[ \overline{\reg}_{\mathsf{CE}}(K) \big]$ and $\E\big[ \overline{\reg}_{\mathsf{CCE}}(K) \big]$. 
\end{theorem}

The proof is provided in \cref{sec:proof_reg_lower_naive}.
\cref{thm:reg_lower_naive} generalizes the lower bound of the single-agent setting \cite[Theorem 3]{fei2020risk} to the multi-agent setting . 
We note that the lower bound in \cref{thm:reg_lower_naive} depends on the largest risk parameter $\beta_*$ (in absolute value) among all agents through the exponential factor $\Phi_H(\beta_*)$.
This yields undesirable implications both in theory and practice.

\paragraph{Theoretical pitfalls.}
The naive regret defined in \eqref{eq:naive_regret} may induce \textit{equilibrium bias}, by which some learned policies only account for the most risk-sensitive agents while overlooking the remaining agents. 
To see this, let us consider the following instance of \MG. Assume $\calS = \{s\}$, $\calA_m = \{g, b\}$ for each $m\in[M]$ and $\{\beta_m\}$ are such that their absolute values are increasing in $m$ (so that $\beta_* = \beta_M$). We denote by $a_{-m}$ a joint action of all agents except agent $m$. For each $m$, we assume $r_{h,m}(s,(g,a_{-m})) = \Phi_H(\beta_*)\frac{1}{\sqrt{K}}$ and $r_{h,m}(s,(b,a_{-m})) = 0$ for all possible $a_{-m}$, and we consider $K \ge \Phi^2_H(\beta_*)$. In particular, for each agent, its reward functions only depend on its own action (and independent of actions of all the other agents). Under this setting, the \MARL problem can be decomposed into $M$ single-agent ones, and an \NE corresponds to each agent executing its own optimal policy in the respective single-agent problem.
Now suppose an algorithm generates $\{\overline{\pi}^k\}$ that incurs   $\overline{\reg}_{\mathsf{NE}}(K) = \Phi_H(\beta_*) \sqrt{KH^2}$ (\eg\ by taking $\overline{\pi}^k_{h,m}(s) = b$ for $m\in [M-1]$ and $\overline{\pi}^k_{h,M}(s) = g$). 
It can be seen that the attained naive regret matches the lower bound in \cref{thm:reg_lower_naive} and the algorithm appears nearly optimal. However, existing results in  \citet{fei2021exponential} show that an exponentially smaller regret, \ie on the order of or smaller than  $\Phi_H(\beta_m) \sqrt{KH^2}$ (since $\Phi_H(\beta)$ increases exponentially in $|\beta|$), can be achieved by applying a single-agent algorithm to each agent individually. Therefore, under the naive regret, the policies $\{\overline{\pi}^k\}$ only perform nearly optimally for the agent with the largest risk sensitivity while being exponentially sub-optimal for all the other agents.

\paragraph{Practical pitfalls.}

While the concept of equilibrium bias has been primarily explored in theoretical contexts, its real-world implications are far-reaching and often detrimental. In the realm of investment, this bias could disproportionately favor the most risk-seeking or risk-averse investors, potentially leading to adverse impacts and instability in economic activities. Similarly, in the world of multiplayer online games, such a bias tends to favor a select few of the most aggressive or passive players, thereby creating an imbalanced gaming environment and diminishing the experience for other participants.

Given the significant shortcomings of equilibrium bias, as highlighted by the naive regret formulation \eqref{eq:naive_regret}, there emerges a compelling need for an alternative performance metric in risk-sensitive \MARL that can more effectively address these issues.

\subsection{Risk-Balanced Regret}

The above discussion
% , in particular the deficiencies of the standard notion of regret, 
motivates us to propose a new notion of regret, which we call \emph{risk-balanced regret}. We provide its definition below.

\begin{definition}
\label{def:regret}
For product policies $\{\pi^k\}_{k=1}^K$, we define the risk-balanced regret  with respect to \NE as 
\begin{equation}
    \reg_{\nash}(K)=\sum_{k\in[K]}\max_{m\in[M]}\frac{(V_{1,m}^{*,\pi_{-m}^{k}}-V_{1,m}^{\pi^{k}})(s_{1})}{\Phi_{H}(\beta_{m})}.
    \label{eq:reg_normalized}
\end{equation}

% This is based on the notion of $(\beta,\eps)$-approximate risk-balanced \NE, \ie, for any $\beta=\{\beta_{m}\}_{m\in[M]}$
% and $\eps>0$, a product policy $\pi$ is a $(\beta,\eps)$-approximate risk-balanced \NE if it satisfies 
% \begin{align*}
%     \max_{m\in[M]}\frac{(V_{1,m}^{*,\pi_{-m}}-V_{1,m}^{\pi})(s_{1})}{\Phi_{H}(\beta_{m})}\le\eps.
% \end{align*}
Moreover, 
% we similarly define the risk-balanced regret for \CE and \CCE. 
% We may similarly define the risk-balanced regret and the $(\beta,\eps)$-approximate risk-balanced equilibria for \CCE and \CE can be similarly defined by normalizing with the factor $\Phi_{H}(\beta_{m})$ for each agent, which we do below for the sake of completeness. 
% More specifically, 
for joint policies $\{\pi^k\}_{k=1}^K$, we define the risk-balanced regret with respect to \CE as
\begin{align*}
    \mathrm{Regret}_{\mathsf{CE}}(K) \defeq \sum_{k \in [K]} \max_{m\in[M]}\max_{\psi\in\Psi_m} \frac{(V_{1,m}^{\psi\diamond\pi^k} - V_{1,m}^{\pi^k})(s_1)}{\Phi_H(\beta_m)},
\end{align*}
% For any joint policies $\{\pi^k\}$, the risk-balanced regret 
and with respect to \CCE as 
\begin{align*}
    \mathrm{Regret}_{\mathsf{CCE}}(K) \defeq \sum_{k \in [K]} \max_{m\in[M]} \frac{(V_{1,m}^{*,\pi_{-m}^k} - V_{1,m}^{\pi^k})(s_1)}{\Phi_H(\beta_m)}.
\end{align*}

\end{definition}

It is important to note that while the regret definitions for \NE and \CCE bear similarities, the definition for \CCE is more inclusive, applying to all joint policies without the necessity of them being product policies, as is mandated in the \NE definition.
The term $\reg_{\nash}(K)$ can be interpreted as a normalization of the regret concept within the single-agent learning paradigm. Here, the risk-balanced factors ${\Phi_H(\beta_m)}$ play a pivotal role, symmetrically moderating the sub-optimality experienced by each agent in a dynamic manner, reflective of their respective risk sensitivities.
In conjunction with these concepts, we also introduce a spectrum of notions pertaining to approximate equilibria, each intricately connected to the framework of risk-balanced regret.
% This notion of sample complexity is closely related to regret.

\begin{definition}
\label{def:apx_eqil}
We say a product policy $\pi$ is $(\beta, \eps)$-approximate \NE if 
\begin{align*}
     \max_{m\in[M]} \frac{(V_{1,m}^{*,\pi_{-m}} - V_{1,m}^{\pi})(s_1)}{\Phi_H(\beta_m)}
     \le \eps.
\end{align*}
In addition, we say a joint policy $\pi$ is $(\beta, \eps)$-approximate \CE if 
\begin{align*} \max_{m\in[M]}\max_{\psi\in\Psi_m} \frac{(V_{1,m}^{\psi\diamond\pi} - V_{1,m}^{\pi})(s_1)}{\Phi_H(\beta_m)}
     \le \eps,
\end{align*}
and $(\beta, \eps)$-approximate \CCE if 
    \begin{align*} \max_{m\in[M]} \frac{(V_{1,m}^{*,\pi_{-m}} - V_{1,m}^{\pi})(s_1)}{\Phi_H(\beta_m)}
     \le \eps.
\end{align*}

\end{definition}

% \begin{remark}
    
    % Perhaps more interestingly, the normalization only appears in the definition of regret and \emph{not} in the algorithm, as we will show in \cref{sec:algo}.
% \end{remark}
Furthermore, a simple relationship between $\reg_{\nash}(K)$ and $\overline{\reg}_{\nash}(K)$ may be observed, as in 
\begin{align}
    \reg_{\nash}(K) &=\sum_{k\in[K]}\max_{m\in[M]}\frac{(V_{1,m}^{*,\pi_{-m}^{k}}-V_{1,m}^{\pi^{k}})(s_{1})}{\Phi_{H}(\beta_{m})} \nonumber \\ 
    &\ge \sum_{k\in[K]}\max_{m\in[M]}\frac{(V_{1,m}^{*,\pi_{-m}^{k}}-V_{1,m}^{\pi^{k}})(s_{1})}{\max_{\ell\in[M]}\Phi_{H}(\beta_{\ell})} \nonumber \\ 
    &= \frac{1}{\Phi_{H}(\beta_{*})} \sum_{k\in[K]}\max_{m\in[M]}(V_{1,m}^{*,\pi_{-m}^{k}}-V_{1,m}^{\pi^{k}})(s_{1}) \nonumber \\ 
    &= \frac{1}{\Phi_{H}(\beta_{*})}\overline{\reg}_{\nash}(K), \label{eq:simple_ineq}
\end{align}

where the inequality holds since $\Phi_H(\beta)$ is increasing in $|\beta|$. It is not hard to see that similar equations to \eqref{eq:simple_ineq} can be derived for $\reg_{\mathsf{CE}}(K)$ and $\reg_{\mathsf{CCE}}(K)$. 
In view of \eqref{eq:simple_ineq}, we have that the risk-balanced regret \eqref{eq:reg_normalized} is more general than the naive regret  \eqref{eq:naive_regret} in the sense that, for any algorithm, attaining an upper bound on $\reg_{\nash}(K)$ implies that it also attains an upper bound on $\overline{\reg}_{\nash}(K)$. Specifically, if $\reg_{\nash}(K) \le U$ for some $U \ge  0$, then we have  $\overline{\reg}_{\nash}(K) \le \Phi_H(\beta_*) \reg_{\nash}(K) \le \Phi_H(\beta_*) U$, thanks to \eqref{eq:simple_ineq}.

Combined with \cref{thm:reg_lower_naive},
the inequality \eqref{eq:simple_ineq} directly leads to the following lower bound for risk-balanced regret.
% This directly implies that our new notion of regret $\reg_{\nash}(K)$ exhibits the desired behavior, \ie, balancing the performance of all agents.

\begin{theorem}
% [Lower Bound]
\label{thm:reg_lower} 
    % Suppose $\left|\beta_*\right|H \geq \log 4$ and $K = \Omega(e^{|\beta|(H-1)}-1)$, then 
    Under the same setting as \cref{thm:reg_lower_naive}, 
    any algorithm obeys the lower bound 
    \begin{align*}
        \E\big[ \reg_{\nash}(K) \big] = \widetilde{\Omega} (\sqrt{K H^2}).
    \end{align*}
    The same result holds for $\E\big[ \reg_{\mathsf{CE}}(K) \big]$ and $\E\big[ \reg_{\mathsf{CCE}}(K) \big]$.
\end{theorem}

We remark that the results of \cref{thm:reg_lower} coincide with those of  \cref{thm:reg_lower_naive} for $M=1$ or when all agents take the same risk parameter.

% \todo{This bound is more general: implies previous bound}

% Notice that the extra factor of $\Phi_{H}(\beta_{m})$ is removed from the lower bound.

Given \cref{thm:reg_lower}, let us discuss how the risk-balanced regret in \cref{def:regret} overcomes the issue of equilibrium bias suffered by the naive regret. We observe that a nearly optimal algorithm under the risk-balanced regret has to learn a policy for \emph{each} agent such that the (un-normalized) regret of agent $m$ is upper bounded by a quantity proportional to risk-balanced factor $\Phi_H(\beta_m)$.
% and coincides with the one in the single-agent setting.
To see this, we let $\reg_{m,\nash}(K) \defeq \sum_{k\in[K]} \frac{(V_{1,m}^{*,\pi_{-m}^{k}}-V_{1,m}^{\pi^{k}})(s_{1})}{\Phi_{H}(\beta_{m})}$ be the (normalized) regret incurred by agent $m$ alone. In order to match with a lower bound (\eg, the single-agent version of \cref{thm:reg_lower}), an algorithm must achieve the near-optimal regret bound $\reg_{m,\nash}(K) = \widetilde{O} (\sqrt{K \cdot \mathrm{poly}(H)})$ for each individual agent $m \in [M]$; otherwise, if for some agent $m'$  the algorithm incurs  $\reg_{m',\nash}(K) = \widetilde{\omega} ({\sqrt{K \cdot \mathrm{poly}(H)}})$, then we would have $\reg_{\nash}(K) \ge \reg_{m',\nash}(K) = \widetilde{\omega} ({\sqrt{K \cdot \mathrm{poly}(H)}})$, which is sub-optimal compared to the lower bound in \cref{thm:reg_lower} and results in a contradiction to our assumption that the algorithm is nearly optimal.
In the remaining paper, we present an algorithm and proves that it nearly attains the lower bound of \cref{thm:reg_lower}.

\section{Algorithm}\label{sec:algo}

In this section, we introduce the \textbf{M}ulti-\textbf{A}gent \textbf{R}isk-\textbf{S}ensitive \textbf{V}alue \textbf{I}teration algorithm, abbreviated as \texttt{MARS-VI}, presented in \cref{alg:MRSA}. This algorithm is designed to estimate \NE, \CE, and \CCE within the context of multi-agent general-sum \MGs, specifically tailored to the entropic risk measure.

\begin{algorithm}[ht]
\begin{algorithmic}[1]
    \Require number of episodes $K$
    
    \Ensure $\pihat$
    
    \State
    % $\epsilon_{\text{gap}}\leftarrow1$
    Set $\Delta_V \leftarrow H$, 
    % \State 
     and initialize $\{N_{h}(s,a)\}_{h\in[H]}$ and $\{N_{h}(s,a,\cdot)\}_{h\in[H]}$
    as zero functions
    
    \For{episode $k=1,2,\ldots,K$}
    
        \State $\forall m\in[M]:\quad\Vup_{H+1,m}(\cdot)\leftarrow0$, $\Vlo_{H+1,m}(\cdot)\leftarrow0$
        
        \For{step $h=H,H-1,\ldots,1$}\emph{ }\label{line:LSVI_estim_value_begin-1}
        
            \State $\forall m\in[M]:\quad\qup_{h,m}(\cdot,\cdot)\leftarrow e^{\beta_{m}(H-h+1)}$,\quad
            $\qlo_{h,m}(\cdot,\cdot)\leftarrow1$
            
            \For{$(m,s,a)\in[M]\times\cS\times\cA$ such that $N_{h}(s,a)\ge1$}
            % \Comment{compute Q-functions}
                
                \State $\Qup_{h,m}(s,a), \Qlo_{h,m}(s,a) \gets$ \texttt{Q-Update}$()$
                % $(k,h,m,s,a)$ 
            
            \EndFor
            
                 \State $\pi_{h}(\cdot|s)\leftarrow\eqsol(
               \{(-1)^{\II(\beta_m < 0)}  e^{\beta_{m}\Qup_{h,m}(s,\cdot)}\}_{m\in[M]})$
                \label{line:MRSA_solve_eq}
                
            \For{$(m,s)\in[M]\times\cS$} 
            % \Comment{solve equilibrium and
            % compute V-functions}

                \State $\Vup_{h,m}(s)\leftarrow\frac{1}{\beta_{m}}\log\{\sum_{a}\pi_{h}(a|s)e^{\beta_{m}\Qup_{h,m}(s,a)}\}$
                
                \State $\Vlo_{h,m}(s)\leftarrow\frac{1}{\beta_{m}}\log\{\sum_{a}\pi_{h}(a|s)e^{\beta_{m}\Qlo_{h,m}(s,a)}\}$
            
            \EndFor \label{line:LSVI_estim_value_end-1}
        
        \EndFor
        
        \If{$\max_{m \in [M]} (\Vup_{1,m} - \Vlo_{1,m} )(s_1) \le \Delta_V$}
            \State $\Delta_V \leftarrow (\Vup_{1,m} - \Vlo_{1,m} )(s_1), \quad 
            \widehat{\pi} \leftarrow
            \pi
            $
        \EndIf
        
        \State Receive $s_{1}$
        
        \For{step $h=1,2,\ldots,H$}
        % \Comment{policy execution}
        \label{line:LSVI_exec_policy_begin-1}
        
        \State Take actions $a_{h}\sim\pi_{h}(\cdot|s_{h})$ and observe
        $r_{h}(s_{h},a_{h})$ and $s_{h+1}$
        
        \State $N_{h}(s_{h},a_{h})\leftarrow N_{h}(s_{h},a_{h})+1$
        
        \State $N_{h}(s_{h},a_{h},s_{h+1})\leftarrow N_{h}(s_{h},a_{h},s_{h+1})+1$
        
        \State $\Phat_{h}(\cdot|s_{h},a_{h})\leftarrow\frac{N_{h}(s_{h},a_{h},\cdot)}{N_{h}(s_{h},a_{h})}$
        
        \EndFor \label{line:LSVI_exec_policy_end-1}
    
    \EndFor
\end{algorithmic}
\caption{Multi-Agent Risk-Sensitive Value Iteration (\texttt{MARS-VI}) \label{alg:MRSA}}
\end{algorithm}

\begin{algorithm}[ht]
\begin{algorithmic}[1]
    \Require 
    % $(k,h,m,s,a)$, and sharing 
    All necessary variables from \cref{alg:MRSA}

    \State $\gamma_{h,m}(s,a)\leftarrow C\left|e^{\beta_{m}(H-h+1)}-1\right|\sqrt{\frac{S\iota}{N_{h}(s,a)}}$
    for some universal constant $C>0$ \label{line:LSVI_bonus}
                
    \State $\qup_{h,m}(s,a)\leftarrow e^{\beta_{m}r_{h,m}(s,a)}[\Phat_{h}e^{\beta_{m}\Vup_{h+1,m}}](s,a)$
    \label{line:MRSA_q_upper}
    
    \State $\qlo_{h,m}(s,a)\leftarrow e^{\beta_{m}r_{h,m}(s,a)}[\Phat_{h}e^{\beta_{m}\Vlo_{h+1,m}}](s,a)$
    \label{line:MRSA_q_lower}
    
    \State Update $\Qup_{h,m}(s,a)$ following \eqref{eqn:MRSA_Q_upper} \label{line:MRSA_Q_upper}
    
    \State Update $\Qlo_{h,m}(s,a)$ following \eqref{eqn:MRSA_Q_lower} \label{line:MRSA_Q_lower}
    
    \State Return $\Qup_{h,m}(s,a)$ and $\Qlo_{h,m}(s,a)$

\end{algorithmic}
\caption{\texttt{Q-Update} \label{alg:Q-update}}
\end{algorithm}

Within each episode $k$, we update the upper and lower confidence bounds $\Qup_{h,m}$ and $\Qlo_{h,m}$ of the action-value function using the estimated transition kernel based on samples collected from the past $k-1$ episodes.
We present the detailed updates of the confidence bounds in \cref{alg:Q-update}.
For each agent $m$, step $h$, and state-action pair $(s,a)$, we compute $\qup_{h,m}(s,a)$ and $\qlo_{h,m}(s,a)$ as estimates of the exponential action-value functions $\E_{s'}[e^{\beta_m[r_{h,m}(s,a)+\Vup_{h+1, m}(s')]}]$ and $\E_{s'}[e^{\beta_m[r_{h,m}(s,a)+\Vlo_{h+1,m}(s')]}]$, respectively.
We then transform  $\qup_{h,m}$ and $\qlo_{h,m}$ in \cref{line:MRSA_q_upper,line:MRSA_q_lower} of \cref{alg:Q-update} into estimates $\Qup_{h,m}$ and $\Qlo_{h,m}$, respectively, by incorporating a bonus term $\gamma_{h,m}$ and applying proper truncation:  
\begin{gather}
    \Qup_{h,m}(s,a)\gets
    \begin{cases}
        \frac{1}{\beta_{m}}\log\{\min\{\qup_{h,m}(s,a)+\gamma_{h,m}(s,a),e^{\beta_{m}(H-h+1)}\}\} & \text{if }\beta_{m}>0;\\
        \frac{1}{\beta_{m}}\log\{\max\{\qup_{h,m}(s,a)-\gamma_{h,m}(s,a),e^{\beta_{m}(H-h+1)}\}\} & \text{if }\beta_{m}<0,
    \end{cases} \label{eqn:MRSA_Q_upper} \\
    \Qlo_{h,m}(s,a)\leftarrow
    \begin{cases}
        \frac{1}{\beta_{m}}\log\{\max\{\qlo_{h,m}(s,a)-\gamma_{h,m}(s,a),1\}\} & \quad\quad\quad\quad\text{if }\beta_{m}>0;\\
        \frac{1}{\beta_{m}}\log\{\min\{\qlo_{h,m}(s,a)+\gamma_{h,m}(s,a),1\}\} & \quad\quad\quad\quad\text{if }\beta_{m}<0.
    \end{cases} \label{eqn:MRSA_Q_lower}
\end{gather}
The bonus term $\gamma_{h,m}$ facilitates Risk-Sensitive Optimism in the Face of Uncertainty \citep{fei2020risk} by augmenting  $\qup_{h,m}$ and  $\qlo_{h,m}$. Notice that the way  $\gamma_{h,m}$ joins the formulae depends on the sign of $\beta_m$ for each agent $m$. For the upper bound $\Qup_{h,m}$, we add bonus for $\beta_{m}>0$ and subtract it for $\beta_{m}<0$, whereas in the lower bound $\Qlo_{h,m}$, we subtract bonus for $\beta_{m}>0$ while adding it for $\beta_{m}<0$.  
This is because the exponential function $z \mapsto e^{\beta z}$ is an increasing function when $\beta > 0$ and decreasing function when $\beta < 0$; so are $\qup_{h,m}$ and $\qlo_{h,m}$ given their construction. 
% due to the relation between value function $V$ and the corresponding exponential value function $e^{\beta V}$; it is a monotone decreasing function in $V$ for $\beta_m<0$. 
Subtracting bonus, \ie, $\qup_{h,m}(s,a)-\gamma_{h,m}(s,a)$, under $\beta<0$ yields a smaller estimation of the exponential value function, and this implies a \emph{larger} upper confidence bound $\Qup_{h,m}$ of the action-value function. 
A similar argument holds for $\Qlo_{h,m}$. 
Another feature of \cref{alg:MRSA} is the difference in thresholding applied to $\Qup_{h,m}$ and $\Qlo_{h,m}$. This is due to the fact that $\qup_{h,m}$ and $\qlo_{h,m}$ are designed to estimate the upper and lower exponential value functions, whose upper and lower bounds are $e^{\beta(H-h+1)}$ and 1 respectively.

% The policy update is handled by an oracle solver for \NE, as shown in \cref{line:MRSA_solve_eq} of \cref{alg:MRSA}.
% In each step $h$ of the algorithm, we aim to find a \NE policy $\pi_{h}$ for the general-sum game in which the payoff matrices are $\{e^{\beta_{m}\Qup_{h,m}(s,\cdot)}\}_{m\in[M]}$, namely, the risk-sensitive game is with respect to the exponential value function estimates. 

In \cref{line:MRSA_solve_eq} of \cref{alg:MRSA}, we update the policy by solving a one-step multi-agent game via the oracle subroutine $\eqsol$, which can be instantiated by existing solvers for \NE, \CE and \CCE,  respectively \citep{berg2017exclusion}.
Here, we assume that $\eqsol$ maximizes the utility of all agents. We therefore use the \textit{signed} exponential value estimates $\{(-1)^{\II(\beta_m < 0)}  e^{\beta_{m}\Qup_{h,m}(s,\cdot)}\}$ as input to the solver;
it can be seen that for $\beta_m < 0$, maximizing with respect to $-  e^{\beta_{m}\Qup_{h,m}(s,\cdot)}$ amounts to  maximizing with respect to $\Qup_{h,m}(s,\cdot)$. 
% This is in  contrast with risk-neutral self-play algorithms in the existing literature   \citep{bai2020provable}, which uses unsigned value estimates as the game value. 
We also note that existing solvers for \gls*{CCE} and \gls*{CE} based on linear programming have polynomial time complexity, while those for \NE are PPAD-hard  \citep{daskalakis2013complexity}.

% In particular, we only consider \NE to avoid redundancy, where $\eqsol$ in \cref{line:MRSA_solve_eq} represents a solver for \NE; the variants of \cref{alg:MRSA} for \gls*{CCE} and \gls*{CE} can be easily acquired by replacing $\eqsol$ with the target subroutines that solve for a one-step multi-agent game. It is worth noting that existing solvers for \gls*{CCE} and \gls*{CE} take polynomial time while those for \NE may not be able to finish in polynomial time \citep{liu2020sharp}.

\section{Main Results}\label{sec:results}

In this section we present our main theoretical results, which comprise of upper bounds on the risk-balanced regret attained by \cref{alg:MRSA} and its induced sample complexity.
% We present a near-optimal regret upper bound for \cref{alg:MRSA} that matches the lower bound in \cref{thm:reg_lower} up to a logarithmic factor. 
% Note that our proofs on \NE also carry over to \CCE and \CE, and the regret upper bounds for \gls*{CCE} and \gls*{CE} share the same order. We keep the related proofs in Appendix for curious readers.

\begin{theorem}
\label{thm:reg_upper} 
    For any $\delta\in(0,1]$, with probability $1-\delta$, Algorithm \ref{alg:MRSA} attains the following regret upper bound when $\eqsol$ is instantiated as a NE solver: 
    \begin{align*}
        \reg_{\nash}(K)
        = \widetilde{O}\big(\sqrt{KH^{4}S^{2}A}\big).
        % \lesssim\sqrt{H^{2}S^{2}AK \log^2(HSAK)}.
    \end{align*}
    The same result holds for $\reg_{\mathsf{CE}}(K) $ and $ \reg_{\mathsf{CCE}}(K) $ when $\eqsol$ is instantiated as a \CE and \CCE solver, respectively.
    Moreover, with probability at least $\frac{2}{3}$, $\pihat$ output by \cref{alg:MRSA} is a $(\beta,\eps)$-approximate \NE, \CE or \CCE if $K = \widetilde{\Omega}(H^4 S^2 A /\eps^{2})$. 
    % (We ignore polynomial factors of $H$, $S$, $A$ for the time being).
\end{theorem}

The proof is provided in \cref{sec:proof_reg_upper}. We make the following remarks for \cref{thm:reg_upper}.

\paragraph{Comparison with the lower bound.}
In view of Theorem \ref{thm:reg_lower}, we see that the above regret upper bound is nearly optimal up to a logarithmic factor in $K$ and  polynomial factors of $H$, $S$ and $A$.
Note that although agents may take different risk parameters $\{\beta_m\}$, the upper bound in Theorem \ref{thm:reg_upper} is not influenced by the difference among $\{\beta_m\}$ as they have been  ``normalized out'' 
% the influence of $\{\beta_{m}\}$ 
through the risk-balanced factor $\{\Phi_H(\beta_m)\}$ in the definition of regret.  
% We remark that the exponential dependency on $H$ in $\Phi_H(\beta)$ cannot be improved by algorithms that are nearly optimal in $K$, and that enables $\reg_{\nash}$ to focus on the part of the regret that could be improved by efficient algorithm design and running such algorithms for more episodes.

% \paragraph{Comparison with lower bound and existing results.}
\paragraph{Comparison with existing results.}
Our results can also be connected back to some of the existing works, showing that our bounds generalize theirs: 
\begin{itemize}
    % \item The regret upper bound of \cref{thm:reg_upper} matches with the lower bound in \cref{thm:reg_lower} with respect to $K$.

    \item 
    When $\beta_{m}\to0$ for all $m\in[M]$, our \cref{thm:reg_upper} recovers the upper bound for \glspl{MG} with risk-neutral agents \citep[Theorem 16]{liu2020sharp}, since $\lim_{b \to 0} \Phi_H(b) = 1$ and $\reg_{\mathsf{NE}}(K)$ tends to  $\text{RN-Regret}_{\mathsf{NE}}(K)$
    
    \item 
    When $M=1$, the multi-agent setting reduces to the single-agent setting, and \cref{thm:reg_upper} implies that an agent with risk parameter $\beta$ incurs the normalized regret (by $\Phi_H(\beta)$) of order  $\widetilde{O}\big(\sqrt{KH^{4}S^{2}A}\big)$, recovering the bound presented in  \citet[Theorem
    1]{fei2021exponential} up to logarithmic factors.
\end{itemize}

\paragraph{Technical highlights.}

% Different agents incur different level of suboptimality, as captured in $\Phi_{H}(\beta_{m})$. 
The proof crucially relies on controlling the differences between upper and lower confidence bounds $\{\Vup^k_{h,m}\}$ and $\{\Vlo^k_{h,m}\}$. In particular, we show that the quantities $\{\frac{e^{\beta_{m}\Vup_{h,m}^{k}} - e^{\beta_{m} \Vlo_{h,m}^{k}}}{\rho_{h,m}}\}$ (where $\rho_{h,m}$ is a  quantity that depends on $h,m$) is upper bounded by a carefully constructed sequence  $\{U_{h}^{k}\}$, such that $U_{1}^{k}(s^k_1)\ge\frac{(V_{1,m}^{*,\pi^k_{-m}}-V_{1,m}^{\pi^k})(s^k_{1})}{\Phi_{H}(\beta_{m})}$ for all $m\in[M]$. We then analyze the evolution of the sequence $\{U_{h}^{k}\}$, which yields the upper bound of regret. Perhaps interestingly, 
% the construction of  $\{U_{h}^{k}\}$ reveals that 
the normalization sequence $\{\rho_{h,m}\}$ evolves in  different ways for risk-seeking and risk-averse agents, thus revealing the inherent asymmetry between the two types of agents. Such finding helps highlight, in a quantitative way, the role of the risk-balanced factors $\{\Phi_H(\beta_m)\}$ in symmetrizing the sub-optimality of each agent in computation of regret. 
% in particular w.r.t.\ the normalization factors $\{\rho_{h,m}\}$. 
% ($\{m:\beta_{m}>0\}$ and $\{m:\beta_{m}<0\}$).  
We believe that this finding could be of independent interest for further research in risk-sensitive multi-agent games.

% U_{h}^{k}(s)\ge\max_{m\in[M]:\beta_{m}>0}\frac{(e^{\beta_{m}\Vup_{h,m}^{k}}-e^{\beta_{m}\Vlo_{h,m}^{k}})(s)}{e^{\beta_{m}(H-h+1)}-e^{-\beta_{m}(h-1)}}\ge0,

\section{Conclusion}
In this paper, we study the problem of  risk-sensitive \MARL under the setting of general-sum \MG, where  agents optimize the entropic risk measure of rewards and possibly take different risk preferences. We demonstrate that a naive definition of regret adapted from risk-neutral \MARL suffers equilibrium bias by inducing  policies that favor the most risk-sensitive agents without taking into account of the other agents in the same game. Motivated by such deficiency of the naive regret, we propose a novel notion of regret, named as risk-balanced regret. We derive a lower bound w.r.t\ risk-balanced regret, from which we show that the proposed regret overcomes the issue of equilibrium bias.
In addition, we propose a self-play \MARL algorithm based on value iteration for learning \NE, \CE and \CCE of the general-sum \MG, and we provide a nearly optimal upper bound for the proposed algorithm w.r.t.\ the risk-balanced regret. The bound is shown to generalize  existing results derived under risk-neutral or single-agent settings.

% As risk-sensitive \MARL is closely related to psychology and behavioral sciences, a closely related area is on \MARL involving mixed-motive agents. 
% Some recent literature has  touched on this aspect of the problem and considered \gls*{MARL} under heterogeneous preferences \citep{koster2020model}. For example, the work of \cite{leibo2017multi,mckee2020social} have discussed mixed-motive games with diverse social preferences and introduced a notion of social value orientation to characterize the intrinsic motivation under the dilemma of reward trade-off. It is also known that naive \gls*{MARL} fail to cooperate and form alliance under zero-sum games and therefore unable to promote social welfare \citep{hughes2020learning}; a few methods have been proposed for heterogeneous agent learning tasks such as common-pool source appropriation and cooperation under causal influence, where agents are required to find socially positive equilibrium \citep{jaques2019social,perolat2017multi}.
% All of these could be interesting venues for future research on risk-sensitive \MARL.

\newpage
\bibliographystyle{ims}
\bibliography{references}

\clearpage
\appendix
\section{Proof of Theorem \ref{thm:reg_lower_naive}} \label{sec:proof_reg_lower_naive}

We provide a regret lower bound that any algorithm has to incur, and this is achieved through considering a hard instance of the $K$-episode and $H$-step \MGs. Let $m_*$ be the index of the most risk-sensitive agent, \ie, $m_* \defeq \argmax_{m \in [M]} |\beta_m|$ with ties broken arbitrarily. 
We assume that the transition kernel of the \MG depends only on the action of $m_*$, and it then reduces to an \MDP with respect to $m_*$.
To further simplify the \MG and remove the effect of all agents other than $m_*$ on $\overline{\reg}_{\mathsf{NE}}(K)$, we define reward functions such that all agents except $m_*$ receive a constant reward no matter which action they take.
More specifically, the \MG has three states: an initial state $s_0$ that serves as a dummy state, an absorbing state $s_1$ in which the agent $m_*$ keeps getting positive rewards, and an absorbing state $s_2$ where the agent $m_*$ gets no reward at all. There are also two actions $a_1,a_2$ available to all agents at every state, \ie, $r_{h,m_*}(s_0,a) = 0$, $r_{h,m_*}(s_1,a) = 1$, and $r_{h,m_*}(s_2,a) = 0$ for the most risk-sensitive agent $m_*$, and $r_{h,m}(s_0,a) = 0$, $r_{h,m}(s_1,a) = 1$, and $r_{h,m}(s_2,a) = 1$ for all other agents $m\neq m_*$ and all $a\in\{a_1,a_2\}$. In other words, the rewards for all agents except agent $m_*$ are always the same and contains no stochasticity.
The transition kernel is simple as it only depends on the action of agent $m_*$ at the initial state $s_0$.
Especially, we let the state transitions from $s_0$ to $s_1$ if agent $m_*$ takes $a_1$ and to $s_2$ if it takes $a_2$.

Recall that the regret is defined as the sum of the regrets on the worst performing agent at each episode, and all agents except $m_*$ incur no regret under the \MG we defined. Thus, we have
\begin{align*}
    \overline{\reg}_{\mathsf{NE}}(K) & = \sum_{k\in[K]}\max_{m\in[M]}(V_{1,m}^{*,\pi_{-m}^{k}}-V_{1,m}^{\pi^{k}})(s_{1}) \\
    & = \sum_{k\in[K]} (V_{1,m_*}^{*,\pi_{-m_*}^{k}}-V_{1,m_*}^{\pi^{k}})(s_{1}) \\
    & = \sum_{k\in[K]} (V_{1,m_*}^* - V_{1,m_*}^{\pi^{k}_{m_*}})(s_{1}).
\end{align*}
Then $\overline{\reg}_{\mathsf{NE}}(K)$ equals to the regret of the most risk-sensitive agent $m_*$, and the \MG is reduced to an \MDP with respect to $m_*$, where $V_{1,m_*}^*$ denotes the optimal value function of $m_*$ and $V_{1,m_*}^{\pi^{k}_{m_*}}$ denotes the value function of $m_*$ under policy $\pi^k$.
To simplify the notation, we will denote them as $V_1^*$ and $V_1^{\pi^k}$ without explicitly referring the agent $m_*$.
The \NE in this case corresponds to agent $m_*$ playing the optimal action with respect to its own \MDP and all the other agents playing arbitrary policy. Therefore, we drop the subscript of \NE and use $\overline{\reg}(K)$ to simplify notation.
Notice that such \MG is equivalent to a $K$-episode and $H$-step \MDP of $m_*$, which is further equivalent to a $K$-round bandit of $m_*$ due to the nature of absorbing states $s_1$ and $s_2$.

In particular, we construct two bandit problems with two arms for the agent $m_*$.
The idea is to construct two bandits, each with a pair of hard-to-distinguish arms.
The first bandit machine has the following two arms: the first arm has reward $H-1$ with probability $p_1$ and reward $0$ with probability $1-p_1$ if $\beta_*>0$ (reward $H-1$ with probability $1-p_1$ and reward $0$ with probability $p_1$ if $\beta_*<0$); the second arm has reward $H-1$ with probability $p_2$ and reward $0$ with probability $1-p_2$ if $\beta_*>0$ (and reward $H-1$ with probability $1-p_2$ and reward $0$ with probability $p_2$ if $\beta_*<0$). The second bandit machine also has two arms: the first arm has reward $H-1$ with probability $q_1$ and reward $0$ with probability $1-q_1$ if $\beta_*>0$ (reward $H-1$ with probability $1-q_1$ and reward $0$ with probability $q_1$ if $\beta_*<0$); the second arm has reward $H-1$ with probability $q_2$ and reward $0$ with probability $1-q_2$ if $\beta_*>0$ (and reward $H-1$ with probability $1-q_2$ and reward $0$ with probability $q_2$ if $\beta_*<0$).
To see the correspondence between the \MDP of $m_*$ and the bandit problem, agent $m_*$ taking action $a_1$ is equivalent to pressing the first arm of the bandit machine; similarly, taking action $a_2$ is equivalent to pressing the second arm.

With proper setup of $p_1,p_2$ and $q_1,q_2$, we are able to show that no policy $\pi$ can do well on both of the bandit problems.
We denote the regret of $\pi$ on the first bandit to be $\overline{\reg}_1(K)$ and that on the second bandit to be $\overline{\reg}_2(K)$. The worst regret between the two can be lower bounded through
\begin{align*}
    \max\{\overline{\reg}_1(K)+\overline{\reg}_2(K)\} & \geq \frac{1}{2}\overline{\reg}_1(K)+\frac{1}{2}\overline{\reg}_2(K).
\end{align*}
Following \cref{lem:lower-bandit} with proper choice on $p_1,p_2$ and $q_1,q_2$, we conclude that the lower bound of $\overline{\reg}(K)$ on the bandit with the worst regret
\begin{align*}
    \overline{\reg}(K) & \gtrsim \frac{1}{2}\left[ \overline{\reg}_1(K) + \overline{\reg}_2(K) \right] \\
    & \gtrsim \frac{e^{|\beta_*|(H-1)}-1}{|\beta_*|} \sqrt{\frac{1}{2} Ke^{-|\beta_*|(H-1)}} \\
    & \overset{(i)}{\gtrsim} \frac{e^{|\beta_*|(H-1)}-1}{|\beta_*|}\sqrt{\frac{K}{\log K}} \\
    & \overset{(ii)}{\gtrsim} \frac{e^{|\beta_*|H}-1}{|\beta_*|}\sqrt{\frac{K}{(\log K)^3}}.
\end{align*}
Here $(i)$ is due to $\log\log K \gtrsim |\beta_*|(H-1)$ and $(ii)$ is due to $e^{|\beta_*|(H-1)} - 1 \geq (1-\frac{1}{H})e^{-|\beta_*|}(e^{|\beta_*|H} - 1)$ and $\log\log K \gtrsim |\beta_*|$. More specifically, notice that we have a decomposition
\begin{align*}
    e^{|\beta_*|}(e^{|\beta_*|(H-1)} - 1) = (e^{|\beta_*|H} - 1) - (e^{|\beta_*|} - 1)
\end{align*}
and it holds that $(e^{|\beta_*|H} - 1) / (e^{|\beta_*|} - 1) \geq \lim_{x\to 1_+} \frac{x^H-1}{x-1} = H$ for any $\beta_*\neq 0$ due to the convexity of the function $x^H$. It then follows that 
\begin{align*}
    e^{|\beta_*|}(e^{|\beta_*|(H-1)} - 1) \geq \Big(1-\frac{1}{H}\Big)(e^{|\beta_*|H} - 1).
\end{align*}

\begin{lemma}\label{lem:lower-bandit}
    We assume the bandit machines in \cref{thm:reg_lower_naive}. For any corresponding \MG with $H \geq 8$ and $K \geq \max\{16e^{|\beta_*|(H-1)}, 16H\}$,
    there exists a set of $p_1,p_2$ and $q_1,q_2$, such that the regret of any policy obeys
    \begin{align*}
        \overline{\reg}_1(K) + \overline{\reg}_2(K) & \gtrsim \frac{e^{|\beta_*|(H-1)}-1}{|\beta_*|} \sqrt{\frac{1}{2} Ke^{-|\beta_*|(H-1)}}.
    \end{align*}
\end{lemma}

\begin{proof}
    The proof follows a similar argument in \citet{fei2020risk} and \citet{fei2022}, and we supply a complete proof here that adapts to the alternative conditions of the lemma.
    For the bandit machines defined in \cref{thm:reg_lower_naive}, if we let $p_2 = e^{-|\beta_*|(H-1)}$ and
    \begin{align*}
        p_1 = q_1 = 
        \begin{cases}
            p_2+ \overline p, & \beta_* > 0; \\
            p_2- \overline p, & \beta_* < 0;
        \end{cases}
        \qquad q_2 = 
        \begin{cases}
            p_2+ 2\overline p, & \beta_* > 0; \\
            p_2- 2\overline p, & \beta_* < 0,
        \end{cases}
    \end{align*}
    then it holds that $p_1,p_2,q_1,q_2\leq \frac{1}{2}$ for any $\overline p \leq \frac{1}{4} e^{-|\beta_*|(H-1)}$ if $|\beta_*|(H-1) \geq \log 4$ and $H \geq 2$.
    If we let $p_2 = \frac{1}{H}$ and
    \begin{align*}
        p_1 = q_1 = 
        \begin{cases}
            p_2+ \overline p, & \beta_* > 0; \\
            p_2- \overline p, & \beta_* < 0;
        \end{cases}
        \qquad q_2 = 
        \begin{cases}
            p_2+ 2\overline p, & \beta_* > 0; \\
            p_2- 2\overline p, & \beta_* < 0,
        \end{cases}
    \end{align*}
    then it holds that $p_1,p_2,q_1,q_2\leq \frac{1}{2}$ for any $\overline p \leq \frac{1}{4H}$ if $|\beta_*|(H-1)\leq \log H$ and $H>8$.
    The following argument holds for both scenarios.
    
    Recall that in the $K$-round bandit problem, any policy $\pi^k$ can only pick one of the two arms that are corresponding to action $a_1$ and $a_2$. Let us define $r_a$ to be the reward from taking action $a\in\{a_1,a_2\}$. Without loss of generality, we look at the first bandit machine and assume that the action $a_1$ corresponds to the optimal arm and $a_2$ corresponds to the sub-optimal arm. The regret can be written as
    \begin{align*}
        \overline{\reg}_1(K) & = \sumk \frac{1}{|\beta_*|} \Big| \log\expect_p e^{\beta_*r_{a_1}} - \log\Big(\prob_p[a^k=a_1]\expect_p e^{\beta_*r_{a_1}} + \prob_p[a^k=a_2]\expect_p e^{\beta_*r_{a_2}}\Big) \Big| \\
        & = \sumk \frac{1}{|\beta_*|} \bigg| \log\frac{\prob_p[a^k=a_1]\expect_p e^{\beta_*r_{a_1}} + \prob_p[a^k=a_2]\expect_p e^{\beta_*r_{a_2}}}{\expect_p e^{\beta_*r_{a_1}}} \bigg|,
    \end{align*}
    where the probability $\prob_p$ is with respect to both the policy $\pi$ and the success probability of the arms. 
    Notice that $\expect_p e^{\beta_*r_{a_1}} \geq \expect_p e^{\beta_*r_{a_2}}$ for $\beta_*>0$, and it follows that
    \begin{align*}
        \bigg| \log\frac{\prob_p[a^k=a_1]\expect_p e^{\beta_*r_{a_1}} + \prob_p[a^k=a_2]\expect_p e^{\beta_*r_{a_2}}}{\expect_p e^{\beta_*r_{a_1}}} \bigg| & \geq \log\Big(\frac{|\expect_p e^{\beta_*r_{a_2}} - \expect_p e^{\beta_*r_{a_1}}|}{\expect_p e^{\beta_*r_{a_1}}} \prob_p[a^k=a_2] + 1\Big).
    \end{align*}
    Therefore, under either setup of $p_1$ and $p_2$, we have
    \begin{align*}
        \overline{\reg}_1(K) & \geq \frac{1}{2|\beta_*|} \frac{|\expect_p e^{\beta_*r_{a_2}} - \expect_p e^{\beta_*r_{a_1}}|}{\expect_p e^{\beta_*r_{a_1}}} \sumk \prob_p[a^k=a_2],
    \end{align*}
    where the inequality follows from $\log(1+x)\geq x/2$ for $x\in[0,1]$ and the assumption that $\overline p \leq \frac{1}{4}p_2$. 
    Swap the order of $a_1$ and $a_2$, we similarly get a lower bound for either pair of $q_1$ and $q_2$ that 
    \begin{align*}
        \overline{\reg}_2(K) \geq \frac{1}{2|\beta_*|} \frac{|\expect_q e^{\beta_*r_{a_1}} - \expect_q e^{\beta_*r_{a_2}}|}{\expect_q e^{\beta_*r_{a_2}}} \sumk \prob_q[a^k=a_1].
    \end{align*}
    In particular, for the first bandit machine we have
    \begin{align*}
        \frac{|\expect_p e^{\beta_*r_{a_2}} - \expect_p e^{\beta_*r_{a_1}}|}{\expect_p e^{\beta_*r_{a_1}}} & = \frac{|(\prob_p[a_1]-\prob_p[a_2])e^{\beta_*(H-1)}-(\prob_p[a_1]-\prob_p[a_2])|}{\prob_p[a_1]e^{\beta_*(H-1)}+(1-\prob_p[a_1])} \\
        & = \frac{|\overline p (e^{\beta_*(H-1)}-1)|}{\prob_p[a_1]e^{\beta_*(H-1)}+(1-\prob_p[a_1])} \\
        & \geq \frac{\overline p}{4} (e^{|\beta_*|(H-1)}-1),
    \end{align*}
    where the inequality holds for both pairs of $p_1$ and $p_2$.
    Similarly, we have 
    \begin{align*}
        \frac{|\expect_q e^{\beta_*r_{a_1}} - \expect_q e^{\beta_*r_{a_2}}|}{\expect_q e^{\beta_*r_{a_2}}} & \geq \frac{\overline p}{4} (e^{|\beta_*|(H-1)}-1)
    \end{align*}
    for both pairs of $q_1$ and $q_2$.
    Hence, we can combine both lower bounds on $\overline{\reg}_1(K)$ and $\overline{\reg}_2(K)$:
    \begin{align*}
        \overline{\reg}_1(K) + \overline{\reg}_2(K) & \geq \frac{\overline p}{4|\beta_*|} (e^{|\beta_*|(H-1)}-1) \cdot \sumk \big(\prob_p[a^k=a_2] + \prob_
        q[a^k=a_1]\big).
    \end{align*}
    It is note-worthy here that $a_2$ is sub-optimal for the first bandit and $a_1$ is sub-optimal for the second bandit.
    Notice that for any policy $\pi$ applied to both bandit machines, we have
    \begin{align*}
        \sumk \big(\prob_p[a^k=a_2] + \prob_
        q[a^k=a_1]\big) & = \expect_p\Big[\sumk \II\{a^k=a_2\}\Big] + \expect_
        q\Big[\sumk \II\{a^k=a_1\}\Big],
    \end{align*}
    where
    \begin{align*}
        \expect_p\Big[\sumk \II\{a^k=a_2\}\Big] & \geq \frac{K}{2} \prob_p\Big[\sumk \II\{a^k=a_2\} > \frac{K}{2}\Big], \\
        \expect_
        q\Big[\sumk \II\{a^k=a_1\}\Big] & \geq \frac{K}{2} \prob_q\Big[\sumk \II\{a^k=a_1\} > \frac{K}{2}\Big]
    \end{align*}
    due to a simple lower bound on the cases where the sub-optimal arm is selected more than half of the time.
    We denote the shorthand $p_{\beta_*} \defeq p_2$ if $\beta_* > 0$ and $p_{\beta_*} \defeq 1-p_2$ if $\beta_* < 0$; $q_{\beta_*} \defeq p_2$ if $\beta_* > 0$ and $q_{\beta_*} \defeq 1-p_2$ if $\beta_* < 0$. 
    A lower bound on the sum of the paired probabilities is given by
    \begin{align*}
    & \prob_p\Big[\sumk \II\{a^k=a_1\}\leq \frac{K}{2}\Big] + \prob_q\Big[\sumk \II\{a^k=a_1\} > \frac{K}{2}\Big] \\
    & \qquad\qquad \geq \frac{1}{2}\exp\left(-\kl\big(\mathrm{Ber}(p_{\beta_*})\|\mathrm{Ber}(q_{\beta_*})\big) K\right),
    \end{align*}
    where it follows from \citet[Theorem 14.2]{lattimore2020bandit} that
    \begin{align*}
        \prob_p\Big[\sumk \II\{a^k=a_1\}\leq \frac{K}{2}\Big] + \prob_q\Big[\sumk \II\{a^k=a_1\} > \frac{K}{2}\Big] \geq \frac{1}{2}\exp(-\kl(\prob_p\|\prob_q))
    \end{align*}
    and \citet[Lemma 15.1]{lattimore2020bandit} that
    \begin{align*}
        \kl(\prob_p\|\prob_q) = \kl\big(\mathrm{Ber}(p_{\beta_*})\|\mathrm{Ber}(q_{\beta_*})\big) \cdot \expect_p\Big[\sumk \II\{a^k=a_2\}\Big].
    \end{align*}
    Following the definition of \gls*{KL} divergence and some simple calculations, we have 
    \begin{align*}
        \kl\big(\mathrm{Ber}(p_{\beta_*})\|\mathrm{Ber}(q_{\beta_*})\big) & \leq \frac{(q_{\beta_*}-p_{\beta_*})^2}{q_{\beta_*}(1-q_{\beta_*})} \\
        & \leq \frac{8\overline p^2}{p_2(1-p_2)},
    \end{align*}
    where the first inequality follows from $\log(1+x) \leq x$ on $\RR$ and the second inequality follows from the definition of the bandit such that $|p_2-q_2| = 2\overline p$ and $p_2\leq q_2\leq \frac{1}{2}$ for $\beta_*>0$ and $\frac{1}{2}p_2\leq q_2\leq p_2\leq \frac{1}{2}$ for $\beta_*<0$.
    Consequently, we have
    \begin{align*}
        \sumk \big(\prob_p[a^k=a_2] + \prob_
        q[a^k=a_1]\big) \geq \frac{K}{4} \exp\Big(-\frac{8K\overline p^2}{p_2(1-p_2)}\Big).
    \end{align*}
    Finally, we combine the lower bounds on $\overline{\reg}_1(K)$ and $\overline{\reg}_2(K)$ together to get
    \begin{align*}
        \overline{\reg}_1(K) + \overline{\reg}_2(K) & \geq \frac{K\overline p}{32|\beta_*|} (e^{|\beta_*|(H-1)}-1) \exp\Big(-\frac{8K\overline p^2}{p_2(1-p_2)}\Big) \\
        & \gtrsim \frac{e^{|\beta_*|(H-1)}-1}{|\beta_*|} \sqrt{\frac{1}{2} Ke^{-|\beta_*|(H-1)}},
    \end{align*}
    where the second inequality follows from taking $\overline p = \sqrt{(p_2(1-p_2))/K}$ for $K\geq 16/p_2$ such that $\overline p \leq \frac{1}{4}p_2$.
\end{proof}

\section{Proof of Theorem \ref{thm:reg_upper} \label{sec:proof_reg_upper}}

We focus on the proof for the case of \NE; the proofs for \CE and \CCE follow the same reasoning, with the differences presented in \cref{sec:proof_reg_ce_cce}. For any state $s \in \calS$, policy $\pi$ and  function $Q': \calS \times \calA \to \mathbb{R}$, we denote $[\calG_{\pi} Q'](s) \defeq \E_{a\sim\pi}[Q'(s,a)]$.

\subsection{Some Useful Lemmas}

Let us first present a uniform concentration result.

\begin{lemma}\label{lem:uni_con} 
    For $G>0$, consider the function class 
    \begin{align*}
    \cW=\{e^{\beta g}\ |\ g:\cS\to[0,G]\}.
    \end{align*}
    For any $\delta\in(0,1]$, and for all $(k,h,s,a)\in[K]\times[H]\times\cS\times\cA$
    with $N_{h}^{k}(s,a)\ge1$, there exists a universal constant $c>0$
    such that, with probability at least $1-\delta$, 
    \begin{align*}
    \left|[(\Phat_{h}^{k}-\P_{h})e^{\beta g}](s,a)\right|\le c\left|e^{\beta G}-1\right|\sqrt{\frac{S\log(HSAK/\delta)}{N_{h}^{k}(s,a)}}.
    \end{align*}
\end{lemma}
\begin{proof}
    Define $\cC_{\eps}(\cW)$ to be an $\eps$-covering of $\cG$
    with respect to the $\ell_{\infty}$ norm for any $\eps>0$. Mathematically,
    this means that for any $W\in\cW$, there exists $W'\in\cC_{\eps}(\cW)$
    such that $\sup_{s\in\cS}\left|W(s)-W'(s)\right|\le\eps$. It
    is not hard to verify that $\left|\cC_{\eps}(\cW)\right|\le\left(3\left|e^{\beta G}-1\right|/\eps\right)^{S}$.
    Now fix a $(k,h,s,a)\in[K]\times[H]\times\cS\times\cA$, and define
    \begin{align*}
        \varphi_{h}^{k,W}(s,a)\coloneqq\frac{1}{N_{h}^{k}(s,a)}\sum_{\tau\in[k-1]}\indic\{(s_{h}^{\tau},a_{h}^{\tau})=(s,a)\}\cdot W(s_{h+1}^{\tau})-(\P_{h}W)(s,a)].
    \end{align*}
    By Hoeffding's inequality and a union bound over both $W\in\cW$ and
    $N_{h}^{k}\in[K]$, with probability at least $1-\delta/(HSA)$, we
    have 
    \begin{align*}
        \left|\sup_{W'\in\cC_{\eps}(\cW)}\varphi_{h}^{k,W'}(s,a)\right|\le\left|e^{\beta G}-1\right|\sqrt{\frac{S\log(3\left|e^{\beta G}-1\right|/\eps)+\log(HSAK/\delta)}{N_{h}^{k}(s,a)}}.
    \end{align*}
    Set $\eps=3\delta\left|e^{\beta G}-1\right|/\sqrt{HSAK}$ and
    the above equation yields 
    \begin{align*}
        \left|\sup_{W'\in\cC_{\eps}(\cW)}\varphi_{h}^{k,W'}(s,a)\right|\le c_{0}\left|e^{\beta G}-1\right|\sqrt{\frac{S\log(HSAK/\delta)}{N_{h}^{k}(s,a)}}.
    \end{align*}
    On the other hand, for any $W\in\cW$, there exists $W'\in\cC_{\eps}(\cW)$
    and a universal constant $c'>0$ such that 
    \begin{align*}
        \left|\varphi_{h}^{k,W}(s,a)-\varphi_{h}^{k,W'}(s,a)\right| & \le2\eps\\
         & \le6\delta\left|e^{\beta G}-1\right|\sqrt{\frac{1}{HSAK}}\\
         & \le c'\left|e^{\beta G}-1\right|\sqrt{\frac{S\log(HSAK/\delta)}{N_{h}^{k}(s,a)}}.
    \end{align*}
    We combine the previous displays and take a union bound over $(h,s,a)\in[H]\times\cS\times\cA$
    to conclude the proof.
\end{proof}
%

% Recall from Algorithm \ref{alg:MRSA} that 
% \begin{align*}
% \gamma_{h,m}^{k}(s,a)\propto(e^{\beta_{m}(H-h+1)}-1).
% \end{align*}
We have the following lemma that bounds the sample exponential value functions.

\begin{lemma}\label{lem:exp_V_Q} 
    For any $\delta\in(0,1]$ and $(k,h,m,s,a)$,
    the following statements hold with probability at least $1-\delta$.
    If $\beta_{m}>0$, then we have 
    \begin{align*}
        e^{\beta_{m}\Vup_{h,m}^{k}(s)}\ge e^{\beta_{m}V_{h,m}^{*,\pi_{-m}^{k}}(s)},\quad e^{\beta_{m}\Vlo_{h,m}^{k}(s)}\le e^{\beta_{m}V_{h,m}^{\pi^{k}}(s)},
    \end{align*}
    \begin{align*}
        e^{\beta_{m}\Qup_{h,m}^{k}(s,a)}\ge e^{\beta_{m}Q_{h,m}^{*,\pi_{-m}^{k}}(s,a)},\quad e^{\beta_{m}\Qlo_{h,m}^{k}(s,a)}\le e^{\beta_{m}Q_{h,m}^{\pi^{k}}(s,a)},
    \end{align*}
    and if $\beta_{m}<0$, we have 
    \begin{align*}
        e^{\beta_{m}\Vup_{h,m}^{k}(s)}\le e^{\beta_{m}V_{h,m}^{*,\pi_{-m}^{k}}(s)},\quad e^{\beta_{m}\Vlo_{h,m}^{k}(s)}\ge e^{\beta_{m}V_{h,m}^{\pi^{k}}(s)},
    \end{align*}
    \begin{align*}
        e^{\beta_{m}\Qup_{h,m}^{k}(s,a)}\le e^{\beta_{m}Q_{h,m}^{*,\pi_{-m}^{k}}(s,a)},\quad e^{\beta_{m}\Qlo_{h,m}^{k}(s,a)}\ge e^{\beta_{m}Q_{h,m}^{\pi^{k}}(s,a)},
    \end{align*}
    % In addition, we have the following recursion: for $\beta_{m}>0$, 
    % \begin{align*}
    %     (e^{\beta_{m}\Qup_{h,m}^{k}}-e^{\beta_{m}\Qlo_{h,m}^{k}})(s,a)\le e^{\beta_{m}r_{h,m}(s,a)}[\Phat_{h}^{k}(e^{\beta_{m}\Vup_{h+1,m}^{k}}-e^{\beta_{m}\Vlo_{h+1,m}^{k}})](s,a)+2\gamma_{h,m}^{k}(s,a)
    % \end{align*}
    % and for $\beta_{m}<0$, 
    % \begin{align*}
    %     (e^{\beta_{m}\Qlo_{h,m}^{k}}-e^{\beta_{m}\Qup_{h,m}^{k}})(s,a)\le e^{\beta_{m}r_{h,m}(s,a)}[\Phat_{h}^{k}(e^{\beta_{m}\Vlo_{h+1,m}^{k}}-e^{\beta_{m}\Vup_{h+1,m}^{k}})](s,a)+2\gamma_{h,m}^{k}(s,a).
    % \end{align*}
\end{lemma}
\begin{proof}
    Let us fix a tuple $(k,h,m,s,a)$ and  $\delta\in(0,1]$. We focus on the inequalities $ e^{\beta_{m}\Qup_{h,m}^{k}(s,a)}\ge e^{\beta_{m}Q_{h,m}^{*,\pi_{-m}^{k}}(s,a)}$ and $e^{\beta_{m}\Vup_{h,m}^{k}(s)}\ge e^{\beta_{m}V_{h,m}^{*,\pi_{-m}^{k}}(s)}$ with  $\beta_{m}>0$; the other inequalities can be established in similar ways.
    
    % Case $\beta_{m}>0$. 
    From the update procedure of the algorithm and
    the Bellman equation, we have the recursion
    \begin{align}
        \quad & (e^{\beta_{m}\Qup_{h,m}^{k}}-e^{\beta_{m}Q_{h,m}^{*,\pi_{-m}^{k}}})(s,a)\nonumber \\
         & =e^{\beta_{m}r_{h,m}(s,a)}[\Phat_{h}^{k}e^{\beta_{m}\Vup_{h+1,m}^{k}}](s,a)+\gamma_{h,m}^{k}(s,a)-e^{\beta_{m}r_{h,m}(s,a)}[\P_{h}e^{\beta_{m}V_{h+1,m}^{*,\pi_{-m}^{k}}}](s,a)\nonumber \\
         \begin{split}
             & =e^{\beta_{m}r_{h,m}(s,a)}[\Phat_{h}^{k}(e^{\beta_{m}\Vup_{h+1,m}^{k}}-e^{\beta_{m}V_{h+1,m}^{*,\pi_{-m}^{k}}})](s,a)+\gamma_{h,m}^{k}(s,a) \\
             & \qquad +[(\Phat_{h}^{k}-\P_{h})e^{\beta_{m}[r_{h,m}(s,a)+V_{h+1,m}^{*,\pi_{-m}^{k}}]}](s,a)
        \end{split}\label{eq:Q_recur}
    \end{align}
    Note that by \cref{lem:uni_con} we have 
    \begin{align}
        -\gamma^k_{h,m}(s,a) \le [(\Phat_{h}^{k}-\P_{h})e^{\beta_{m}[r_{h,m}(s,a)+V_{h+1,m}^{*,\pi_{-m}^{k}}]}](s,a)
        \le
        \gamma^k_{h,m}(s,a).
        \label{eq:bound_estim_error}
    \end{align}
    The remaining proof proceeds by induction. We first check the base case: by definition, $\Vup_{H+1,m}^{k}=V_{H+1,m}^{*,\pi_{-m}^{k}}=0$
    and so $e^{\beta_{m}\Vup_{H+1,m}^{k}}=e^{\beta_{m}V_{H+1,m}^{*,\pi_{-m}^{k}}}$.
    Now suppose $e^{\beta_{m}\Vup_{h+1,m}^{k}}\ge e^{\beta_{m}V_{h+1,m}^{*,\pi_{-m}^{k}}}$
    for some $h\in[H-1]$. Then by \eqref{eq:Q_recur} and 
    % Lemma \ref{lem:uni_con},
    \eqref{eq:bound_estim_error},
    we have 
    \begin{equation}
        e^{\beta_{m}\Qup_{h,m}^{k}(s,a)}\ge e^{\beta_{m}Q_{h,m}^{*,\pi_{-m}^{k}}(s,a)}.\label{eq:Q_dom}
    \end{equation}
    By the construction of \cref{alg:MRSA} and the definition of best response, we have
    \begin{align*}
        e^{\beta_{m}\Vup_{h,m}^{k}(s)}=[\calG_{\pi_{h}^{k}}e^{\beta_{m}\Qup_{h,m}^{k}}](s)=\max_{\nu}[\calG_{\nu\times\pi_{h,-m}^{k}}e^{\beta_{m}\Qup_{h,m}^{k}}](s)
    \end{align*}
    and 
    \begin{align*}
        e^{\beta_{m}V_{h,m}^{*,\pi_{-m}^{k}}(s)}=\max_{\nu}[\calG_{\nu\times\pi_{h,-m}^{k}}e^{\beta_{m}Q_{h,m}^{*,\pi_{-m}^{k}}}](s).
    \end{align*}
    It follows from \eqref{eq:Q_dom} that 
    \begin{align*}
        e^{\beta_{m}\Vup_{h,m}^{k}(s)}\ge e^{\beta_{m}V_{h,m}^{*,\pi_{-m}^{k}}(s)}.
    \end{align*}
    This completes the induction. 
    % From \ref{eq:Q_recur} and on the event
    % of Lemma \ref{lem:uni_con}, we also have 
    % \begin{align*}
    %     e^{\beta_{m}\Qup_{h,m}^{k}(s,a)}-e^{\beta_{m}Q_{h,m}^{*,\pi_{-m}^{k}}(s,a)}.
    % \end{align*}
    % The remaining proof follows the same reasoning.
\end{proof}
% Recall from Algorithm \ref{alg:MRSA} that 
% \begin{align*}
% \gamma_{h,m}^{k}(s,a)\propto(e^{\beta_{m}(H-h+1)}-1).
% \end{align*}
% We have the following lemma. 
The next lemma bounds the difference of $Q$-functions by the bonus and the difference of $V$-functions.

\begin{lemma}\label{lem:Q_V_bonus} 
    For any $\delta\in(0,1]$ and $(k,h,m,s,a)$,
    the following statements hold with probability at least $1-\delta$.
    If $\beta_{m}>0$, 
    \begin{align*}
    (e^{\beta_{m}\Qup_{h,m}^{k}}-e^{\beta_{m}\Qlo_{h,m}^{k}})(s,a)\le e^{\beta_{m}r_{h,m}(s,a)}[\Phat_{h}^{k}(e^{\beta_{m}\Vup_{h+1,m}^{k}}-e^{\beta_{m}\Vlo_{h+1,m}^{k}})](s,a)+2\gamma_{h,m}^{k}(s,a)
    \end{align*}
    and if $\beta_{m}<0$, 
    \begin{align*}
    (e^{\beta_{m}\Qlo_{h,m}^{k}}-e^{\beta_{m}\Qup_{h,m}^{k}})(s,a)\le e^{\beta_{m}r_{h,m}(s,a)}[\Phat_{h}^{k}(e^{\beta_{m}\Vlo_{h+1,m}^{k}}-e^{\beta_{m}\Vup_{h+1,m}^{k}})](s,a)+2\gamma_{h,m}^{k}(s,a).
    \end{align*}
\end{lemma}

% \subsection{Key lemma on sequence of lower bounds}

\begin{proof}
For $\beta_m > 0$, 
% We focus on the case of $\beta_m$, and the other case holds similarly.
the update procedure of Algorithm \ref{alg:MRSA} implies that for 
all $(k,h,m,s,a)$, we have
\begin{align*}
    e^{\beta_{m}\Qup_{h,m}^{k}(s,a)} & \le e^{\beta_{m}r_{h,m}(s,a)}[\Phat_{h}^{k}e^{\beta_{m}\Vup_{h+1,m}^{k}}](s,a)+\gamma_{h,m}^{k}(s,a),\\
    e^{\beta_{m}\Qlo_{h,m}^{k}(s,a)} & \ge e^{\beta_{m}r_{h,m}(s,a)}[\Phat_{h}^{k}e^{\beta_{m}\Vlo_{h+1,m}^{k}}](s,a)-\gamma_{h,m}^{k}(s,a).
\end{align*}
Combining the above displayed equations yields the result. The proof for the case of $\beta_m < 0$ holds similarly.
% This implies the following recursion: for $\beta_{m}>0$, 
% \begin{align*}
%     (e^{\beta_{m}\Qup_{h,m}^{k}}-e^{\beta_{m}\Qlo_{h,m}^{k}})(s,a) & \le e^{\beta_{m}r_{h,m}(s,a)}[\Phat_{h}^{k}(e^{\beta_{m}\Vup_{h+1,m}^{k}}-e^{\beta_{m}\Vlo_{h+1,m}^{k}})](s,a)+2\gamma_{h,m}^{k}(s,a)\\
%      & \le e^{\beta_{m}}[\Phat_{h}^{k}(e^{\beta_{m}\Vup_{h+1,m}^{k}}-e^{\beta_{m}\Vlo_{h+1,m}^{k}})](s,a)+2\gamma_{h,m}^{k}(s,a)
% \end{align*}
% and for $\beta_{m}<0$, 
% \begin{align*}
%     (e^{\beta_{m}\Qlo_{h,m}^{k}}-e^{\beta_{m}\Qup_{h,m}^{k}})(s,a) & \le e^{\beta_{m}r_{h,m}(s,a)}[\Phat_{h}^{k}(e^{\beta_{m}\Vlo_{h+1,m}^{k}}-e^{\beta_{m}\Vup_{h+1,m}^{k}})](s,a)+2\gamma_{h,m}^{k}(s,a)\\
%      & \le[\Phat_{h}^{k}(e^{\beta_{m}\Vlo_{h+1,m}^{k}}-e^{\beta_{m}\Vup_{h+1,m}^{k}})](s,a)+2\gamma_{h,m}^{k}(s,a)
% \end{align*}
\end{proof}

 Let 
\begin{equation}
    W_{h}^{k}(s,a) \defeq \min\{1,[\Phat_{h}^{k}U_{h+1}^{k}](s,a)+2z_{h}^{k}(s,a)\}\label{eq:dom_seq_1}
\end{equation}
where 
% for $\beta_{m}>0$,
% \begin{align*}
%     z_{h}^{k}(s,a)\coloneqq\max_{m\in[M]}\frac{\gamma_{h,m}^{k}(s,a)}{e^{\beta_{m}(H-h+1)}-e^{-\beta_{m}(h-1)}},
% \end{align*}
% and for $\beta_{m}<0$, 
% \begin{align*}
%     z_{h}^{k}(s,a)\coloneqq\max_{m\in[M]}\frac{\gamma_{h,m}^{k}(s,a)}{1-e^{\beta_{m}H}}.
% \end{align*}
% Putting together, 
we set 
\begin{equation}
    z_{h}^{k}(s,a)\coloneqq\max\left\{ \max_{m\in[M]:\beta_{m}>0}\frac{\gamma_{h,m}^{k}(s,a)}{e^{\beta_{m}(H-h+1)}-e^{-\beta_{m}(h-1)}},\max_{m\in[M]:\beta_{m}<0}\frac{\gamma_{h,m}^{k}(s,a)}{1-e^{\beta_{m}H}}\right\} \label{eq:agg_bonus}
\end{equation}
We let 
\begin{align*}
    U_{H+1}^{k}(s) \defeq 0.
\end{align*}
and for $h\in[H]$, 
\begin{equation}
    U_{h}^{k}(s) \defeq (\calG_{\pi_{h}^{k}}W_{h}^{k})(s)\label{eq:dom_seq_2}
\end{equation}

The next lemma controls $\{U^k_h\}$ through upper and lower bounds.

\begin{lemma}\label{lem:lower_dom_seq}
    For all $(k,h,s)\in[K]\times[H]\times\cS$,
    we have $U_{h}^{k}(s)\le1$, 
    \begin{align*}
        U_{h}^{k}(s)\ge\max_{m\in[M]:\beta_{m}>0}\frac{(e^{\beta_{m}\Vup_{h,m}^{k}}-e^{\beta_{m}\Vlo_{h,m}^{k}})(s)}{e^{\beta_{m}(H-h+1)}-e^{-\beta_{m}(h-1)}}\ge0,
    \end{align*}
     and 
    \begin{align*}
        U_{h}^{k}(s)\ge\max_{m\in[M]:\beta_{m}<0}\frac{(e^{\beta_{m}\Vlo_{h,m}^{k}}-e^{\beta_{m}\Vup_{h,m}^{k}})(s)}{1-e^{\beta_{m}H}}\ge0.
    \end{align*}
\end{lemma}

\begin{proof}
    Let us write $[M]=\cM^{+}\cup\cM^{-}$ where 
    \begin{align*}
        \cM^{+} & \coloneqq\{m\in[M]:\beta_{m}>0\}\\
        \cM^{-} & \coloneqq\{m\in[M]:\beta_{m}<0\}.
    \end{align*}
    
    \textbf{Case I.} We first prove the lemma w.r.t.\ $\cM^{+}$. We first
    verify the base case. Since $U_{H+1}^{k}=\Vup_{H+1,m}^{k}=\Vlo_{H+1,m}^{k}=0$,
    we have 
    \begin{align*}
        1\ge U_{H+1}^{k}(s)=0=\max_{m\in\cM^{+}}\frac{(e^{\beta_{m}\Vup_{H+1,m}^{k}}-e^{\beta_{m}\Vlo_{H+1,m}^{k}})(s)}{1-e^{-\beta_{m}H}}.
    \end{align*}
    Now assume that our claim holds for $U_{h+1}^{k}$ for some $h\in[H-1]$.
    We can deduce
    \begin{align}
        [\Phat_{h}^{k}U_{h+1}^{k}](s,a) & \ge\left[\Phat_{h}^{k}\max_{m\in\cM^{+}}\frac{e^{\beta_{m}\Vup_{h+1,m}^{k}}-e^{\beta_{m}\Vlo_{h+1,m}^{k}}}{e^{\beta_{m}(H-h)}-e^{-\beta_{m}h}}\right](s,a)\nonumber\\
         & =\left[\Phat_{h}^{k}\max_{m\in\cM^{+}}\frac{e^{\beta_{m}}(e^{\beta_{m}\Vup_{h+1,m}^{k}}-e^{\beta_{m}\Vlo_{h+1,m}^{k}})}{e^{\beta_{m}}(e^{\beta_{m}(H-h)}-e^{-\beta_{m}h})}\right](s,a)\nonumber\\
         & =\left[\Phat_{h}^{k}\max_{m\in\cM^{+}}\frac{e^{\beta_{m}}(e^{\beta_{m}\Vup_{h+1,m}^{k}}-e^{\beta_{m}\Vlo_{h+1,m}^{k}})}{e^{\beta_{m}(H-h+1)}-e^{-\beta_{m}(h-1)}}\right](s,a)\nonumber\\
         & \ge\max_{m\in\cM^{+}}\left[\Phat_{h}^{k}\frac{e^{\beta_{m}}(e^{\beta_{m}\Vup_{h+1,m}^{k}}-e^{\beta_{m}\Vlo_{h+1,m}^{k}})}{e^{\beta_{m}(H-h+1)}-e^{-\beta_{m}(h-1)}}\right](s,a) \label{eq:U_interm_bound}
    \end{align}
    We claim that 
    \begin{equation}
        W_{h}^{k}\ge\max_{m\in\cM^{+}}\left[\frac{(e^{\beta_{m}\Qup_{h,m}^{k}}-e^{\beta_{m}\Qlo_{h,m}^{k}})(s,a)}{e^{\beta_{m}(H-h+1)}-e^{-\beta_{m}(h-1)}}\right].\label{eq:W_lb}
    \end{equation}
    If $[\Phat_{h}^{k}U_{h+1}^{k}](s,a)+2z_{h}^{k}(s,a)\ge1$, then \eqref{eq:dom_seq_1}
    implies $W_{h}^{k}=1$ and \eqref{eq:W_lb} is verified; otherwise,
    we have 
    \begin{align*}
        W_{h}^{k} 
        & \overset{(i)}{=}[\Phat_{h}^{k}U_{h+1}^{k}](s,a)+2z_{h}^{k}(s,a)\\
         & \overset{(ii)}{\ge}\max_{m\in\cM^{+}}\left[\Phat_{h}^{k}\frac{e^{\beta_{m}}(e^{\beta_{m}\Vup_{h+1,m}^{k}}-e^{\beta_{m}\Vlo_{h+1,m}^{k}})}{e^{\beta_{m}(H-h+1)}-e^{-\beta_{m}(h-1)}}\right](s,a)+2\max_{m\in\cM^{+}}\frac{\gamma_{h,m}^{k}(s,a)}{e^{\beta_{m}(H-h+1)}-e^{-\beta_{m}(h-1)}}\\
         & \ge\max_{m\in\cM^{+}}\left[\frac{e^{\beta_{m}}[\Phat_{h}^{k}(e^{\beta_{m}\Vup_{h+1,m}^{k}}-e^{\beta_{m}\Vlo_{h+1,m}^{k}})](s,a)+2\gamma_{h,m}^{k}(s,a)}{e^{\beta_{m}(H-h+1)}-e^{-\beta_{m}(h-1)}}\right]\\
         & \ge\max_{m\in\cM^{+}}\left[\frac{(e^{\beta_{m}\Qup_{h,m}^{k}}-e^{\beta_{m}\Qlo_{h,m}^{k}})(s,a)}{e^{\beta_{m}(H-h+1)}-e^{-\beta_{m}(h-1)}}\right],
    \end{align*}
    where step ($i$) holds by \eqref{eq:dom_seq_1}; step ($ii$) holds by \eqref{eq:U_interm_bound}; the last step
    follows from \cref{lem:Q_V_bonus} as well as the facts that $e^{\beta_{m}\Vup_{h+1,m}^{k}} \ge e^{\beta_{m}\Vlo_{h+1,m}^{k}}$ (implied by \cref{lem:exp_V_Q}) and that $e^{\beta_{m}} \ge e^{\beta_{m} r_{h,m}(s,a)}$ (since $r_{h,m}(s,a) \in [0,1]$). 
    Hence, \eqref{eq:W_lb} is verified.
    Since \eqref{eq:dom_seq_1} implies $W_{h}^{k}\le1$, from \eqref{eq:dom_seq_2}
    we have $U_{h}^{k}(s)\le1$; on the other hand, 
    \begin{align*}
        U_{h}^{k}(s) & \overset{(i)}{=}[\calG_{\pi_{h}^{k}}W_{h}^{k}](s)\\
         & \overset{(ii)}{\ge}\left[\calG_{\pi_{h}^{k}}\max_{m\in\cM^{+}}\left[\frac{e^{\beta_{m}\Qup_{h,m}^{k}}-e^{\beta_{m}\Qlo_{h,m}^{k}}}{e^{\beta_{m}(H-h+1)}-e^{-\beta_{m}(h-1)}}\right]\right](s)\\
         & \ge\max_{m\in\cM^{+}}\left[\calG_{\pi_{h}^{k}}\frac{e^{\beta_{m}\Qup_{h,m}^{k}}-e^{\beta_{m}\Qlo_{h,m}^{k}}}{e^{\beta_{m}(H-h+1)}-e^{-\beta_{m}(h-1)}}\right](s)\\
         & =\max_{m\in\cM^{+}}\frac{(e^{\beta_{m}\Vup_{h,m}^{k}}-e^{\beta_{m}\Vlo_{h,m}^{k}})(s)}{e^{\beta_{m}(H-h+1)}-e^{-\beta_{m}(h-1)}}\ge0,
    \end{align*}
    where step ($i$) holds by \eqref{eq:dom_seq_2}; step ($ii$) holds
    by \eqref{eq:W_lb};  the last step follows from the update procedure
    of Algorithm \ref{alg:MRSA}. The induction is completed.
    
    \vspace{0.4em}
    
    \textbf{Case II.} We prove the lemma w.r.t.\ $\cM^{-}$. We first verify
    the base case. Since $\Vup_{H+1,m}^{k}=\Vlo_{H+1,m}^{k}=0$, we have
    \begin{align*}
        1\ge U_{H+1}^{k}(s)=\max_{m\in\cM^{-}}\frac{(e^{\beta_{m}\Vlo_{h,m}^{k}}-e^{\beta_{m}\Vup_{h,m}^{k}})(s)}{1-e^{\beta_{m}H}}=0.
    \end{align*}
    Assume that our claim holds for $U_{h+1}^{k}$. We can deduce
    \begin{align*}
         \quad[\Phat_{h}^{k}U_{h+1}^{k}](s,a)
         & \ge\left[\Phat_{h}^{k}\max_{m\in\cM^{-}}\frac{e^{\beta_{m}\Vlo_{h+1,m}^{k}}-e^{\beta_{m}\Vup_{h+1,m}^{k}}}{1-e^{\beta_{m}H}}\right](s,a)\\
         & \ge\max_{m\in\cM^{-}}\left[\Phat_{h}^{k}\frac{e^{\beta_{m}\Vup_{h+1,m}^{k}}-e^{\beta_{m}\Vlo_{h+1,m}^{k}}}{1-e^{\beta_{m}H}}\right](s,a).
    \end{align*}
    We claim that 
    \begin{equation}
        W_{h}^{k}\ge\max_{m\in\cM^{-}}\left[\frac{(e^{\beta_{m}\Qlo_{h,m}^{k}}-e^{\beta_{m}\Qup_{h,m}^{k}})(s,a)}{1-e^{\beta_{m}H}}\right]. \label{eq:W_lb_M_minus}
    \end{equation}
    If $[\Phat_{h}^{k}U_{h+1}^{k}](s,a)+2z_{h}^{k}(s,a)\ge1$, then \eqref{eq:dom_seq_1}
    implies $W_{h}^{k}=1$ and \eqref{eq:W_lb_M_minus} is verified; otherwise,  
    \begin{align*}
        W_{h}^{k}  &=[\Phat_{h}^{k}U_{h+1}^{k}](s,a)+2z_{h}^{k}(s,a)\\
         & \overset{(i)}{\ge}\max_{m\in\cM^{-}}\left[\Phat_{h}^{k}\left(\frac{e^{\beta_{m}\Vlo_{h+1,m}^{k}}-e^{\beta_{m}\Vup_{h+1,m}^{k}}}{1-e^{\beta_{m}H}}\right)\right](s,a)+2\max_{m\in\cM^{-}}\frac{\gamma_{h,m}^{k}(s,a)}{1-e^{\beta_{m}H}}\\
         & \ge\max_{m\in\cM^{-}}\left[\frac{[\Phat_{h}^{k}(e^{\beta_{m}\Vlo_{h+1,m}^{k}}-e^{\beta_{m}\Vup_{h+1,m}^{k}})](s,a)}{1-e^{\beta_{m}H}}+2\frac{\gamma_{h,m}^{k}(s,a)}{1-e^{\beta_{m}H}}\right]\\
         & \ge\max_{m\in\cM^{-}}\left[\frac{(e^{\beta_{m}\Qlo_{h,m}^{k}}-e^{\beta_{m}\Qup_{h,m}^{k}})(s,a)}{1-e^{\beta_{m}H}}\right],
    \end{align*}
    where step ($i$) holds by the induction hypothesis on $U^k_{h+1}$ and \eqref{eq:agg_bonus}, and the last step
    follows from \cref{lem:Q_V_bonus}. This verifies our claim for
    $W_{h}^{k}$. We also have 
    \begin{align*}
        U_{h}^{k}(s) & =[\calG_{\pi_{h}^{k}}W_{h}^{k}](s)\\
         & \ge\left[\calG_{\pi_{h}^{k}}\max_{m\in\cM^{-}}\left[\frac{e^{\beta_{m}\Qlo_{h,m}^{k}}-e^{\beta_{m}\Qup_{h,m}^{k}}}{1-e^{\beta_{m}H}}\right]\right](s)\\
         & \ge\max_{m\in\cM^{-}}\calG_{\pi_{h}^{k}}\left[\frac{e^{\beta_{m}\Qlo_{h,m}^{k}}-e^{\beta_{m}\Qup_{h,m}^{k}}}{1-e^{\beta_{m}H}}\right](s)\\
         & =\max_{m\in\cM^{-}}\frac{(e^{\beta_{m}\Vlo_{h,m}^{k}}-e^{\beta_{m}\Vup_{h,m}^{k}})(s)}{1-e^{\beta_{m}H}},
    \end{align*}
    where the first step holds by \eqref{eq:dom_seq_2}. The induction
    is completed.
\end{proof}

\subsection{Controlling $U_1^k$}

Let us define 
\begin{align*}
    \Delta_{h}^{k} &\coloneqq U_{h}^{k}(s_{h}^{k}), \\
% \end{align*}
% \begin{align*}
    \zeta_{h}^{k} &\coloneqq(\calG_{\pi^{k}}W_{h}^{k})(s_{h}^{k})-W_{h}^{k}(s_{h}^{k},a_{h}^{k}), \\
% \end{align*}
% \begin{align*}
    \overline p_{h}^{k} &\coloneqq(\P_{h}U_{h+1}^{k})(s_{h}^{k},a_{h}^{k})-U_{h+1}^{k}(s_{h+1}^{k}).
\end{align*}
We have 
\begin{align*}
    U_{h}^{k}(s_{h}^{k}) & =(\calG_{\pi^{k}}W_{h}^{k})(s_{h}^{k})\\
     & =\zeta_{h}^{k}+W_{h}^{k}(s_{h}^{k},a_{h}^{k})\\
     & \overset{(i)}{\le}\zeta_{h}^{k}+2z_{h}^{k}+(\Phat_{h}^{k}U_{h+1}^{k})(s_{h}^{k},a_{h}^{k})\\
     & =\zeta_{h}^{k}+2z_{h}^{k}+[(\Phat_{h}^{k}-\P_{h})U_{h+1}^{k}](s_{h}^{k},a_{h}^{k})+[\P_{h}U_{h+1}^{k}](s_{h}^{k},a_{h}^{k})\\
     & \overset{(ii)}{\le}\zeta_{h}^{k}+3z_{h}^{k}+[\P_{h}U_{h+1}^{k}](s_{h}^{k},a_{h}^{k})\\
     & =\zeta_{h}^{k}+3z_{h}^{k}+\overline p_{h}^{k}+U_{h+1}^{k}(s_{h+1}^{k}),
\end{align*}
where step $(i)$ holds by \eqref{eq:dom_seq_1} and step $(ii)$
follows from Lemma \ref{lem:uni_con}. Recursing and summing over
$k\in[K]$ gives 
\begin{equation}
    \sum_{k\in[K]}U_{1}^{k}(s_{1}^{k})\le\sum_{k\in[K]}\sum_{h\in[H]}(\zeta_{h}^{k}+3z_{h}^{k}+\overline p_{h}^{k}).\label{eq:sum_U_interm}
\end{equation}
Note that $\{\zeta_{h}^{k}\}$ and $\{\overline p_{h}^{k}\}$ are martingale difference
sequences. 
% with respect to some appropriately constructed filtrations.
Also, by \eqref{eq:dom_seq_1} and Lemma \ref{lem:lower_dom_seq},
we have that $\left|\zeta_{h}^{k}\right|,\left|\overline p_{h}^{k}\right|\le1$
for $(k,h)\in[K]\times[H]$. Therefore, the Azuma-Hoeffding inequality
implies that with probability at least $1-\delta$, 
\begin{align*}
    \left|\sum_{k\in[K]}\sum_{h\in[H]}\zeta_{h}^{k}\right| & \le\sqrt{2HK\log(1/\delta)},\\
    \left|\sum_{k\in[K]}\sum_{h\in[H]}\overline p_{h}^{k}\right| & \le\sqrt{2HK\log(1/\delta)}.
\end{align*}
On the other hand, recalling $z_{h}^{k}$ defined in \eqref{eq:agg_bonus} and the construction of $\gamma^k_{h,m}$ in \cref{alg:MRSA},
we have 
\begin{align*}
    z_{h}^{k}\le\sqrt{\frac{S\iota}{N_{h}^{k-1}(s_{h}^{k},a_{h}^{k})}}.
\end{align*}
This implies 
\begin{align*}
    \sum_{k\in[K]}\sum_{h\in[H]}z_{h}^{k} & \le \sum_{h\in[H]}\sum_{k\in[K]}\sqrt{\frac{S\iota}{\min\{1,N_{h}^{k-1}(s_{h}^{k},a_{h}^{k})\}}}\\
     & \overset{(i)}{\le}\sum_{h\in[H]}\sqrt{K\sum_{k\in[K]}\frac{S\iota}{\min\{1,N_{h}^{k-1}(s_{h}^{k},a_{h}^{k})\}}}\\
     & \overset{(ii)}{\le}\sum_{h\in[H]}\sqrt{SK\iota\sum_{(s,a)\in\cS\times\cA}\sum_{k\in[K]}\frac{1}{\min\{1,N_{h}^{k-1}(s,a)\}}}\\
     & \overset{(iii)}{\le}\sum_{h\in[H]}\sqrt{SK\iota\cdot SA(1+\log K)}\\
     & =2\sqrt{H^{2}S^{2}AK\iota^{2}},
\end{align*}
where step $(i)$ follows from the Cauchy-Schwarz inequality, step
$(ii)$ holds by the pigeonhole principle, and step $(iii)$ holds
by the fact that $\sum_{k\in[K]}\frac{1}{k}\le\log K$. Combining
the above results with \eqref{eq:sum_U_interm}, we have 
\begin{equation}
    \sum_{k\in[K]}U_{1}^{k}(s_{1}^{k})\lesssim\sqrt{HK\iota}+\sqrt{H^{2}S^{2}AK\iota^{2}}\lesssim\sqrt{H^{2}S^{2}AK\iota^{2}}.\label{eq:sum_U}
\end{equation}

\subsection{Putting All Together}

By Lemma \ref{lem:lower_dom_seq}, note that for $\beta_{m}>0$, 
\begin{align*}
    U_{1}^{k}(s)\ge\frac{(e^{\beta_{m}\Vup_{1,m}^{k}}-e^{\beta_{m}\Vlo_{1,m}^{k}})(s)}{e^{\beta_{m}H}-1}
\end{align*}
and for $\beta_{m}<0$, 
\begin{align*}
    U_{1}^{k}(s)\ge\frac{(e^{\beta_{m}\Vlo_{1,m}^{k}}-e^{\beta_{m}\Vup_{1,m}^{k}})(s)}{1-e^{\beta_{m}H}},
\end{align*}
which together imply 
\begin{equation}
    U_{1}^{k}(s)\ge\max_{m\in[M]}\frac{(e^{\beta_{m}\Vup_{1,m}^{k}}-e^{\beta_{m}\Vlo_{1,m}^{k}})(s)}{e^{\beta_{m}H}-1}.\label{eq:U_ub}
\end{equation}

Note that on the event of Lemma \ref{lem:exp_V_Q}, we have 
\begin{equation}
    \frac{(e^{\beta_{m}V_{1,m}^{*,\pi_{-m}^{k}}}-e^{\beta_{m}V_{1,m}^{\pi^{k}}})(s_{1}^{k})}{e^{\beta_{m}H}-1}\le\frac{(e^{\beta_{m}\Vup_{1,m}^{k}}-e^{\beta_{m}\Vlo_{1,m}^{k}})(s_{1}^{k})}{e^{\beta_{m}H}-1},
    \label{eq:exp_V_dom}
\end{equation}
% It is worth checking that 
where the above two quantities are non-negative
since for $\beta_{m}>0$, we have 
\begin{align*}
    e^{\beta_{m}V_{1,m}^{*,\pi_{-m}^{k}}}-e^{\beta_{m}V_{1,m}^{\pi^{k}}}\ge0,\qquad e^{\beta_{m}H}-1 > 0,
\end{align*}
and for $\beta_{m}<0$, we have 
\begin{align*}
    e^{\beta_{m}V_{1,m}^{*,\pi_{-m}^{k}}}-e^{\beta_{m}V_{1,m}^{\pi^{k}}}\le0,\qquad e^{\beta_{m}H}-1 < 0.
\end{align*}

We now connect the above results with the regret of  \eqref{eq:reg_normalized}.
For each $m\in[M]$, it is clear that $V_{1,m}^{*,\pi_{-m}^{k}}\ge V_{1,m}^{\pi^{k}}$
by definition. If $\beta_{m}>0$, we have 
\begin{align}
    \frac{(V_{1,m}^{*,\pi_{-m}^{k}}-V_{1,m}^{\pi^{k}})(s_{1}^{k})}{\Phi_{H}(\beta_{m})} & =\frac{H}{e^{\beta_{m}H}-1}\left[\log\left(e^{\beta_{m}V_{1,m}^{*,\pi_{-m}^{k}}(s_{1}^{k})}\right)-\log\left(e^{\beta_{m}V_{1,m}^{\pi^{k}}(s_{1}^{k})}\right)\right]\nonumber\\
     & \overset{(i)}{\le}\frac{H}{e^{\beta_{m}H}-1}\left(e^{\beta_{m}V_{1,m}^{*,\pi_{-m}^{k}}(s_{1}^{k})}-e^{\beta_{m}V_{1,m}^{\pi^{k}}(s_{1}^{k})}\right)\nonumber\\
     & \overset{(ii)}{\le}\frac{H}{e^{\beta_{m}H}-1}\left(e^{\beta_{m}\Vup_{1,m}^{k}(s_{1}^{k})}-e^{\beta_{m}\Vlo_{1,m}^{k}(s_{1}^{k})}\right)\nonumber\\
     & \le H\cdot U_{1}^{k}(s_{1}^{k}), \label{eq:regret_bound_last_pos_beta}
\end{align}
where step $(i)$ holds since $f(x)=\log x$ is 1-Lipschitz for $x\ge1$,
step $(ii)$ holds on the event of Lemma \ref{lem:exp_V_Q}, and the
last step holds by \eqref{eq:U_ub}. If $\beta_{m}<0$, we have 
\begin{align}
    \frac{(V_{1,m}^{*,\pi_{-m}^{k}}-V_{1,m}^{\pi^{k}})(s_{1}^{k})}{\Phi_{H}(\beta_{m})} & =\frac{H}{e^{-\beta_{m}H}-1}\left[\log\left(e^{\beta_{m}V_{1,m}^{\pi^{k}}(s_{1}^{k})}\right)-\log\left(e^{\beta_{m}V_{1,m}^{*,\pi_{-m}^{k}}(s_{1}^{k})}\right)\right]\nonumber\\
     & \overset{(i)}{\le}\frac{He^{-\beta_{m}H}}{e^{-\beta_{m}H}-1}\left(e^{\beta_{m}V_{1,m}^{\pi^{k}}(s_{1}^{k})}-e^{\beta_{m}V_{1,m}^{*,\pi_{-m}^{k}}(s_{1}^{k})}\right)\nonumber\\
     & =\frac{H}{1-e^{\beta_{m}H}}\left(e^{\beta_{m}V_{1,m}^{\pi^{k}}(s_{1}^{k})}-e^{\beta_{m}V_{1,m}^{*,\pi_{-m}^{k}}(s_{1}^{k})}\right)\nonumber\\
     & \overset{(ii)}{\le}\frac{H}{1-e^{\beta_{m}H}}\left(e^{\beta_{m}\Vlo_{1,m}^{k}(s_{1}^{k})}-e^{\beta_{m}\Vup_{1,m}^{k}(s_{1}^{k})}\right)\nonumber\\
     & =\frac{H}{e^{\beta_{m}H}-1}\left(e^{\beta_{m}\Vup_{1,m}^{k}(s_{1}^{k})}-e^{\beta_{m}\Vlo_{1,m}^{k}(s_{1}^{k})}\right)\nonumber\\
     & \le H\cdot U_{1}^{k}(s_{1}^{k}),\label{eq:regret_bound_last_neg_beta}
\end{align}
where step $(i)$ holds since $f(x)=\log x$ is $(e^{-\beta_{m}H})$-Lipschitz
for $x\in[e^{\beta_{m}H},1]$, step $(ii)$ holds on the event of
Lemma \ref{lem:exp_V_Q}, and the last step holds by \eqref{eq:U_ub}.
The proof for the regret upper bound is completed by applying   \eqref{eq:sum_U} to \eqref{eq:regret_bound_last_pos_beta} and \eqref{eq:regret_bound_last_neg_beta} combined.

Finally, a sample complexity guarantee can be derived based on the regret bound, following an argument similar to that presented in \citet{jin2018q}. In particular, we have
\begin{align*}
    \reg_{\nash}(K) & = \sum_{k\in[K]}\max_{m\in[M]}\frac{(V_{1,m}^{*,\pi_{-m}^{k}}-V_{1,m}^{\pi^{k}})(s_{1})}{\Phi_{H}(\beta_{m})} \lesssim\sqrt{H^{4}S^{2}AK\iota^{2}},
\end{align*}
and a random policy $\pi^\ddagger$ defined as a uniform sample from $\{\pi^k\}_{k=1}^K$ enjoys
\begin{align*}
    \max_{m\in[M]} \frac{(V_{1,m}^{*,\pi_{-m}^\ddagger}-V_{1,m}^{\pi^\ddagger})(s_{1})}{\Phi_{H}(\beta_{m})} \lesssim \sqrt{3H^{4}S^{2}A\iota^{2}/K}
\end{align*}
with probability at least $\frac{2}{3}$. It follows that \cref{alg:MRSA} finds $\eps$-optimal policy with $K = \widetilde O(H^4 S^2 A /\eps^2)$ episodes.

\subsection{Proofs for \CCE and \CE \label{sec:proof_reg_ce_cce}}

We supply the proofs of upper and lower  bounds of exponential value functions involved in \cref{alg:MRSA} for \CCE and \CE in this section, which complement the analogous results on \NE in \cref{lem:exp_V_Q}. 
% These lemmas lead to the regret upper bound of \cref{alg:MRSA} when equipped with equilibrium solver on \CCE and \CE, respectively.
The regret upper bounds for \CCE and \CE then follow from substituting \cref{lem:exp_V_Q} with \cref{lem:exp_V_Q_CCE,lem:exp_V_Q_CE}, respectively, and replacing the exponential value function $e^{\beta_m\Vdkhm(s)}$ under the optimal response with $e^{\beta_m\max_{\psi\in\Psi_m}\Vdiamondpikhm(s)}$ for \CE.

\begin{lemma}[Upper and Lower Bounds for \CCE] \label{lem:exp_V_Q_CCE}
    For any $\delta\in(0,1]$ and $(k,h,m,s,a)$, the following statements hold with probability at least $1-\delta$ for any  policy $\pi^k$. If $\beta_m > 0$, we have
    \begin{gather*}
        e^{\beta_m\Vukhm(s)} \geq e^{\beta_m\Vdkhm(s)}, \qquad
        e^{\beta_m\Vpikhm(s)} \geq e^{\beta_m\Vlkhm(s)}, \\
        e^{\beta_m\Qukhm(s,a)}
        \geq e^{\beta_m\Qdkhm(s,a)}, \qquad
        e^{\beta_m\Qpikhm(s,a)} \geq e^{\beta_m\Qlkhm(s,a)};
    \end{gather*}
    and if $\beta_m < 0$, we have
    \begin{gather*}
        e^{\beta_m\Vdkhm(s)} \geq e^{\beta_m\Vukhm(s)}, \qquad
        e^{\beta_m\Vlkhm(s)} \geq e^{\beta_m\Vpikhm(s)}, \\
        e^{\beta_m\Qdkhm(s,a)}
        \geq e^{\beta_m\Qukhm(s,a)}, \qquad
        e^{\beta_m\Qlkhm(s,a)} \geq e^{\beta_m\Qpikhm(s,a)}.
    \end{gather*}
\end{lemma}

\begin{proof}
    We prove this result through induction. We only show the argument for the case of $\beta_m > 0$, and we can multiply $-1$ to the argument below for $\beta_m < 0$. 
    For any $\delta\in(0,1]$, fixed tuple $(k,h,m,s,a)$, and $\beta_m > 0$, if we assume that the value functions at the next step satisfies $e^{\beta_m\Vukhhm(s)} - e^{\beta_m\Vdkhhm(s)} \geq 0$, then we have
    \begin{align*}
        & e^{\beta_m\Qukhm(s)} - e^{\beta_m\Qdkhm(s)} \\
        =\ &  (e^{\beta_m\rhm(s,a)}[\Phatkh e^{\beta_m\Vukhhm}](s,a) - e^{\beta_m\rhm(s,a)}[\Ph e^{\beta_m\Vdkhhm}](s,a)) + \gammakhm(s,a) \\
        =\ &  e^{\beta_m\rhm(s,a)}[\Phatkh (e^{\beta_m\Vukhhm} - e^{\beta_m\Vdkhhm})](s,a) \\
        \ & +  [(\Phatkh-\Ph) e^{\beta_m(\rhm(s,a) + \Vdkhhm)}](s,a) + \gammakhm(s,a) \\
        \geq\ & 0,
    \end{align*}
    where the first component is non-negative due to induction assumption, the sum of the second and the third component is also non-negative due to \cref{lem:uni_con}.
    Notice that $\VukHHm(s) = \VdkHHm(s) = 0$ for all state $s\in\calS$, and $e^{\beta_m\VukHHm(s)} - e^{\beta_m\VdkHHm(s)} \geq 0$ always holds for step $H+1$. Consequently, we have
    \begin{align*}
        e^{\beta_m\Qukhm(s,a)}
        \geq e^{\beta_m\Qdkhm(s,a)}
    \end{align*}
    for all agents, state-action pairs, and $h\in[H]$.
    
    Further, given the assumption $e^{\beta_m\Qukhm(s,a)} - e^{\beta_m\Qdkhm(s,a)} \geq 0$ for any fixed tuple $(k,h,m,s,a)$, it follows that
    \begin{align*}
        e^{\beta_m\Vukhm(s)} - e^{\beta_m\Vdkhm(s)} & = [\DDpikh e^{\beta_m\Qukhm}](s) - e^{\beta_m\Vdkhm(s)} \geq 0,
    \end{align*}
    where by the definition of \CCE equilibrium that 
    \begin{align*}
        [\DDpikh e^{\beta_m\Qukhm}](s) = \max_\nu[\DDnukh e^{\beta_m\Qukhm}](s) \geq e^{\beta_m\Vdkhm(s)}
    \end{align*}
    for any $\beta_m>0$ and 
    \begin{align*}
        [\DDpikh e^{\beta_m\Qukhm}](s) = \min_\nu[\DDnukh e^{\beta_m\Qukhm}](s) \leq e^{\beta_m\Vdkhm(s)}
    \end{align*}
    for any $\beta_m<0$.
    
    Recursion proofs for $(-1)^{\II\{\beta_m<0\}}\cdot(e^{\beta_m\Vpikhm(s)} - e^{\beta_m\Vlkhm(s)}) \geq 0$ and $(-1)^{\II\{\beta_m<0\}}\cdot(e^{\beta_m\Qpikhm(s,a)} - e^{\beta_m\Qlkhm(s,a)}) \geq 0$ follow the same reasoning.
\end{proof}

\begin{lemma}[Upper and Lower Bounds for \CE] \label{lem:exp_V_Q_CE}
    For any $\delta\in(0,1]$ and $(k,h,m,s,a)$, the following statements hold with probability at least $1-\delta$ for any \CE policy $\pi^k$. If $\beta_m > 0$, we have
    \begin{gather*}
        e^{\beta_m\Vukhm(s)} \geq e^{\beta_m\max_{\psi\in\Psi_m}\Vdiamondpikhm(s)}, \qquad
        e^{\beta_m\Vpikhm(s)} \geq e^{\beta_m\Vlkhm(s)}, \\
        e^{\beta_m\Qukhm(s,a)} \geq e^{\beta_m\max_{\psi\in\Psi_m}\Qdiamondpikhm(s,a)}, \qquad
        e^{\beta_m\Qpikhm(s,a)} \geq e^{\beta_m\Qlkhm(s,a)};
    \end{gather*}
    if $\beta_m < 0$, we have
    \begin{gather*}
        e^{\beta_m\max_{\psi\in\Psi_m}\Vdiamondpikhm(s)} \geq e^{\beta_m\Vukhm(s)}, \qquad
        e^{\beta_m\Vlkhm(s)} \geq e^{\beta_m\Vpikhm(s)}, \\
        e^{\beta_m\max_{\psi\in\Psi_m}\Qdiamondpikhm(s,a)} \geq e^{\beta_m\Qukhm(s,a)}, \qquad
        e^{\beta_m\Qlkhm(s,a)} \geq e^{\beta_m\Qpikhm(s,a)}.
    \end{gather*}
\end{lemma}

\begin{proof}
    We prove this result with induction. We only show the argument for the case of $\beta_m>0$, and for $\beta_m<0$ we can multiply $-1$ to the argument below. 
    For any $\delta\in(0,1]$, fixed tuple $(k,h,m,s,a)$, and $\beta_m > 0$, if $e^{\beta_m\Vukhhm(s)} - e^{\beta_m\max_{\psi\in\Psi_m}\Vdiamondpikhhm(s)} \geq 0$, then we have
    \begin{align*}
        & e^{\beta_m\Qukhm(s,a)} - e^{\beta_m\max_{\psi\in\Psi_m}\Qdiamondpikhm(s,a)} \\
        =\ &  (e^{\beta_m\rhm(s,a)}[\Phatkh e^{\beta_m\Vukhhm}](s,a) - e^{\beta_m\rhm(s,a)}[\Ph e^{\beta_m\max_{\psi\in\Psi_m}\Vdiamondpikhhm}](s,a)) + \gammakhm(s,a) \\
        =\ &  e^{\beta_m\rhm(s,a)}[\Phatkh (e^{\beta_m\Vukhhm} - e^{\beta_m\max_{\psi\in\Psi_m}\Vdiamondpikhhm})](s,a) \\
        \ & +  [(\Phatkh-\Ph) e^{\beta_m(\rhm(s,a) + \max_{\psi\in\Psi_m}\Vdiamondpikhhm)}](s,a) + \gammakhm(s,a) \\
        \geq\ & 0,
    \end{align*}
    where the first component is non-negative due to induction assumption, the sum of the second and the third component is also non-negative due to \cref{lem:uni_con}.
    Notice that $\VukHHm(s) = \VdiamondpikHHm(s) = 0$ for all state $s\in\calS$, and $e^{\beta_m\VukHHm(s)} - e^{\beta_m\max_{\psi\in\Psi_m}\VdiamondpikHHm(s)} \geq 0$ holds for step $H+1$.
    Consequently, we have
    \begin{align*}
        e^{\beta_m\Qukhm(s,a)} \geq e^{\beta_m\max_{\psi\in\Psi_m}\Qdiamondpikhm(s,a)}
    \end{align*}
    for all agents, state-action pairs, and $h\in[H]$.
    
    Further, given the assumption $e^{\beta_m\Qukhm(s,a)} - e^{\beta_m\max_{\psi\in\Psi_m}\Qdiamondpikhm(s,a)} \geq 0$ for any fixed tuple $(k,h,m,s,a)$, we have
    \begin{align*}
        e^{\beta_m\Vukhm(s)} - e^{\beta_m\max_{\psi\in\Psi_m}\Vdiamondpikhhm(s)} & = [\DDpikh e^{\beta_m\Qukhm}](s) - e^{\beta_m\max_{\psi\in\Psi_m}\Vdiamondpikhhm(s)} \\
        & \geq 0,
    \end{align*}
    where by the definition of \CE that 
    \begin{align*}
        [\DDpikh e^{\beta_m\Qukhm}](s) = \max_\psi[\calG_{\psi(\pi_h^k)} e^{\beta_m\Qukhm}](s) \geq e^{\beta_m\max_{\psi\in\Psi_m}\Vdiamondpikhhm(s)}
    \end{align*}
    for any $\beta_m>0$ and 
    \begin{align*}
        [\DDpikh e^{\beta_m\Qukhm}](s) = \min_\psi[\calG_{\psi(\pi_h^k)} e^{\beta_m\Qukhm}](s) \leq e^{\beta_m\max_{\psi\in\Psi_m}\Vdiamondpikhhm(s)}
    \end{align*}
    for any $\beta_m<0$. Recall that $e^{\beta_m\max_{\psi\in\Psi_m}\Vdiamondpikhhm(s)} = \max_\psi[\calG_{\psi(\pi_h^k)} e^{\beta_m\max_{\psi'} Q_{h,m}^{\psi'(\pi_h^k)}}](s)$ for $\beta_m>0$ and $e^{\beta_m\max_{\psi\in\Psi_m}\Vdiamondpikhhm(s)} = \min_\psi[\calG_{\psi(\pi_h^k)} e^{\beta_m\max_{\psi'} Q_{h,m}^{\psi'(\pi_h^k)}}](s)$ for $\beta_m<0$.
    
    The recursion arguments for the other two inequalities follow the same reasoning.
\end{proof}

\paragraph*{Future directions and broad impact.} 

With the recent developments in variance-aware learning, an exciting avenue for future exploration is examining how noise-adaptive algorithms \citep{xu2023noise} can enhance risk-sensitive \RL. There is also significant potential in developing efficient algorithms for risk-sensitive RL in varied contexts, such as matching and finding competitive equilibrium in macroeconomics \citep{min2022learn,xu2023finding}. Furthermore, given the close connection between risk-sensitive RL and human learning behaviors, it would be fascinating to investigate its integration with fields like meta-learning and bio-inspired learning, as discussed in studies by \citet{xu2021meta} and \citet{song2021convergence}. Additionally, a critical area for future research lies in exploring the role of risk sensitivity in augmenting unsupervised learning algorithms, a concept touched upon by \citet{ling2019landscape}.

\end{document}